\documentclass[11pt,a4paper,final]{article}
\pdfoutput=1
\usepackage[utf8]{inputenc}
\usepackage{amsmath}
\usepackage{amsfonts}
\usepackage{amssymb}
\usepackage{graphicx}
\usepackage{xcolor}
\usepackage[left=1in,right=1in,top=1in,bottom=1in]{geometry}
\usepackage[breakable]{tcolorbox}

\usepackage{booktabs} % For formal tables
\usepackage[ruled, noend]{algorithm2e}
\usepackage{tikz}
\usetikzlibrary{arrows}
\usetikzlibrary{shapes.multipart}
\usepackage{caption}

\SetAlFnt{\small}
\SetAlCapFnt{\small}
\SetAlCapNameFnt{\small}
\SetAlCapHSkip{0pt}
\IncMargin{-\parindent}
\usepackage{natbib}
\setcitestyle{authoryear}

\usepackage[T1]{fontenc}    % use 8-bit T1 fonts
\usepackage{hyperref}       % hyperlinks
\newcommand{\declarecolor}[2]{\definecolor{#1}{RGB}{#2}\expandafter\newcommand\csname #1\endcsname[1]{\textcolor{#1}{##1}}}

\declarecolor{White}{255, 255, 255}
\declarecolor{Black}{0, 0, 0}
\declarecolor{Maroon}{128, 0, 0}
\declarecolor{Coral}{255, 127, 80}
\declarecolor{Red}{182, 21, 21}
\declarecolor{LimeGreen}{50, 205, 50}
\declarecolor{DarkGreen}{0, 90, 0}
\declarecolor{Navy}{0, 0, 128}

\definecolor{plotblue}{HTML}{377eb8}
\definecolor{plotorange}{HTML}{ff7f00}
\definecolor{plotgreen}{HTML}{4daf4a}

\hypersetup{
    pdftitle={Adaptive Determinantal Client Scheduling in Federated Learning},
    pdfauthor={Wen Xu, Ben Liang, Gary Boudreau, Hamza Sokun},
    colorlinks=true,
    citecolor=DarkGreen,
    linkcolor=Navy,
    filecolor=magenta,      
    urlcolor=black,
    pdfpagemode=UseNone,
}
\usepackage{url}            % simple URL typesetting
\usepackage{booktabs}       % professional-quality tables
\usepackage{nicefrac}       % compact symbols for 1/2, etc.
\usepackage{microtype}      % microtypography
\usepackage{mathtools}
\usepackage{amsthm}
\usepackage{bbm}
\usepackage{tikz} 
\usepackage{wrapfig}
\usepackage{floatrow}
\usepackage{enumitem}
\usepackage{authblk}
\usepackage[capitalise,noabbrev]{cleveref}
\usepackage{booktabs}

\usepackage[suppress]{color-edits}
\addauthor{kh}{red}
\addauthor{iws}{red}
\addauthor{pa}{cyan}

\usepackage[ruled]{algorithm2e}
\usepackage{algorithmic}
\usepackage{caption}
\usepackage{subcaption}
\usepackage{thm-restate}

\usepackage{titlesec}

\titlespacing*{\paragraph}
{0pt}      % left indent
{1.0ex}    % space before
{0.8em}    % space after (before the text begins)

\newtheorem{theorem}{Theorem}
\newtheorem{corollary}{Corollary}

\newtheorem{assumption}{Assumption}

\DeclareMathOperator*{\argmax}{arg\,max}

\newcommand{\diag}{\text{diag}}

\newcommand{\qualityattatn}{q_{n}^{(t)}}

\newcommand{\qualityatt}{\boldsymbol{q}^{(t)}}

\newcommand{\qualityattatone}{q_{1}^{(t)}}
\newcommand{\qualityattatN}{q_{N}^{(t)}}

\newcommand{\learnqualityattatn}{q_{n, \text{learn}}^{(t)}}

\newcommand{\commqualityattatn}{q_{n, \text{comm}}^{(t)}}

\newcommand{\weight}{\boldsymbol{p}}

\newcommand{\weightatn}{p_{n}}

\newcommand{\model}{\boldsymbol{w}}

\newcommand{\modelset}{\mathcal{W}}

\newcommand{\modelatt}{\boldsymbol{w}^{(t)}}
\newcommand{\modelatT}{\boldsymbol{w}^{(T)}}
\newcommand{\modelattatn}{\boldsymbol{w}_{n}^{(t)}}
\newcommand{\modelattplusone}{\boldsymbol{w}^{(t+1)}}

\newcommand{\bZ}{\boldsymbol{Z}}
\newcommand{\bs}{\boldsymbol{s}}
\newcommand{\bv}{\boldsymbol{v}}

\newcommand{\cN}{{\mathcal N}}
\newcommand{\cO}{{\mathcal O}}

\newcommand{\cS}{{\mathcal S}}

\begin{document}

\title{Adaptive Determinantal Client Scheduling in Federated Learning}
\author[1]{Wen Xu}
\author[1]{Ben Liang}
\author[2]{Gary Boudreau}
\author[2]{Hamza Sokun}

\affil[1]{University of Toronto, Canada}
\affil[2]{Ericsson, Canada}
\date{}

\maketitle
\vspace*{-0.4in}

\begin{abstract}
Scheduling clients for model training is critical in federated learning due to both data and system heterogeneity. Most previous works focus on the quality of the scheduled clients to achieve faster convergence, shorter wall-clock convergence time, or better average model performance. They rarely consider the diversity of clients, which is important to counter heterogeneity and improve performance for the worst-off clients. In this work, we advocate the use of determinantal point processes (DPPs) to model and enhance the diversity in client scheduling. We first design the kernel matrices of DPPs using gradient information and quality scores, which inherently enables a flexible quality-diversity trade-off. Applying fast MAP inference over DPPs, we propose \underline{A}daptive \underline{D}eterminantal \underline{C}lient \underline{S}cheduling  (\textsc{ADCS}) in FL. We further quantify the gradient approximation error of \textsc{ADCS} and develop convergence analysis for general biased client selection in FL with non-convex loss functions. We conduct comparative numerical experiments showing that \textsc{ADCS} outperforms state-of-the-art client scheduling algorithms, including both quality-based and diversity-based ones.
\end{abstract}

\section{Introduction}\label{sec:intro}
Federated learning (FL) has become a dominant distributed machine learning paradigm in mobile and edge computing~\citep{FL21_survey}. It is often performed over a large number of edge devices, i.e., clients, such as mobile phones or IoT devices, with the assistance of a central server~\citep{MLSys19_Bonawitz}. This can lead to a high level of heterogeneity among the devices, in terms of their data and hardware. 

Many algorithms have been proposed for FL since the advent of the original \textsc{FedAvg}~\citep{AISTATS17_McMahan}. Most FL algorithms follow a common pattern in each training round: i) client scheduling at the server, ii) model broadcasting from the server to the selected clients, iii) local computation at selected clients, iv) model or gradient transmission from the clients to the server, and v) model aggregation at the server. The first step, i.e., client scheduling at the server, is essential since the number of available clients can be large in real-world applications, and selecting all clients in all rounds can be inadvisable or even prohibitive. The stragglers in computation or communication can greatly impede the wall-clock training time~\citep{FL21_survey}. Furthermore, na\"ive partial client participation in FL is often sub-optimal: The data distributions among the clients are typically not identically distributed. Without carefully addressing the heterogeneity issue, a na\"ive optimization method does not produce favorable convergence behavior. 

Most previous methods consider the following optimization formulation for client scheduling:
\begin{align}\label{DPP:eqn:quality}
    \max_{\cS^{(t)}: |\cS^{(t)}| = m} \; \sum_{n \in \cS^{(t)}} \qualityattatn,
\end{align}
where $\qualityattatn$ is some quality score of client $n$ in round $t$ and $\cS^{(t)}$ is the subset of clients to be scheduled in round $t$, whose cardinality is $m$. For example, \textsc{PowerOfChoice}~\citep{AISTATS22_Cho} chooses $\qualityattatn$ to be the local loss of client $n$ evaluated by the current global model $\modelatt$; optimal client sampling (\textsc{OCS})~\citep{arXiv21_CHR} chooses $\qualityattatn$ as the gradient norm of client $n$ at round $t$; and uniform client scheduling~\citep{AISTATS17_McMahan} can be interpreted as setting $\qualityattatn$ to a constant value and performing random tie breaking. However,~\eqref{DPP:eqn:quality} only accounts for additive individual quality terms, which fails to capture the structure of data heterogeneity among clients. For example, data distributions among the clients can have cluster structures, leading to the tendency for high-quality clients to come from the same clusters. Thus, the quality-only formulation in~\eqref{DPP:eqn:quality} tends to select redundant clients while leaving clients with a broader spectrum of data distributions unused.

To address this problem, in this work we consider both the diversity and quality in FL client scheduling, with an aim to improve the performance of learning for the worst-off clients. We model diversity by determinantal point processes (DPPs)~\citep{AAP75_Macchi, FTML12_Kulesza} and maximize a combination of both diversity and quality in client scheduling. This optimization problem is challenging since a na\"ive algorithm has to go through all possible subsets of clients of size $m$, which is exponential and thus impractical. It is also challenging to analyze how adding diversity consideration impacts the convergence of FL since the resulting client scheduling is often statistically biased.

We mitigate these challenges with the following contributions.
\begin{itemize}

    \item We formulate client scheduling in FL as a unified optimization problem that jointly captures diversity and quality scores across client updates, which recovers quality-only and diversity-only policies as special cases and allows tuning their tradeoff. We propose Adaptive Determinantal Client Scheduling (\textsc{ADCS}), which schedules clients via greedy MAP inference on a DPP kernel constructed from client update similarity and quality scores in each training round. 
    \item We quantify the gradient approximation bias of~\textsc{ADCS} under general cluster structures for the client data distributions. Since considering diversity leads to statistically biased client sampling, we also derive a novel convergence bound for FL performance under biased client scheduling for non-convex loss functions. This provides performance guarantee for \textsc{ADCS}.
    \item We conduct numerical experiments to show that \textsc{ADCS} can achieve significant gains on FL test accuracy for the worst-off clients, compared with state-of-the-art client scheduling strategies. Meanwhile, it maintains competitive average accuracy. 
\end{itemize}

\section{Related Work}\label{DPP:sec:related}
\textbf{Client scheduling in FL.} Communication overhead and data heterogeneity are two major issues in FL since both the communication conditions and the underlying data distributions can be highly heterogeneous when the number of clients is large. To reduce communication overhead and mitigate data heterogeneity, numerous client scheduling strategies have been proposed. They include selecting clients for partial participation uniformly at random or via probabilities proportional to the local dataset sizes~\citep{AISTATS17_McMahan}, selecting clients based on the norm of the gradient updates~\citep{arXiv21_CHR}, selecting clients based on the values of local loss functions~\citep{AISTATS22_Cho}, and selecting clients via solving submodular maximization problems {to explicitly account for client diversity, as in \textsc{DivFL}}~\citep{ICLR22_Balakrishnan}. Besides the common considerations of the learning performance and the optimization convergence with respect to iterations, other factors such as wall-clock training time and resource allocation are widely considered in FL, especially in wireless FL (see for example~\citep{ICC19_Nishio,TWC21_Xu,INFOCOM22_Luo,IOTJ22_Yu,TVT24_Zheng,TOMPECS25_Xu}).

{\textbf{DPP in FL.}} Initially developed for modeling fermions, the DPP is a fundamental tool in quantum physics~\citep{AAP75_Macchi}. It has recently received significant attention in broader contexts due to its elegant theory and favorable empirical performance~\citep{FTML12_Kulesza}. DPPs are commonly used in scenarios where an informative and diverse subset needs to be selected from a large ground set. Typical applications include clustering~\citep{NeurIPS13_Kang}, statistical inference~\citep{JRSS15_Lavancier}, cellular networks~\citep{TCOM15_Li}, recommender systems~\citep{NeurIPS18_Chen}, generative models~\citep{ICML19_Elfeki}, and image processing~\citep{JIS21_Launay}.

Several prior works leverage DPPs in FL~\citep{INFOCOM21_Li,ICASSP23_Zhang,DCOSS24_Bastola}. In~\citep{INFOCOM21_Li}, DPP is used only in a pre-training client-selection stage, while the subsequent round-wise dynamics are handled by a separate importance-based mechanism. In~\citep{ICASSP23_Zhang}, a k-DPP is constructed once from one-shot client profiles and then reused throughout training. Thus different client groups can be selected in different rounds, but under a fixed kernel that does not depend on the evolving model. In~\citep{DCOSS24_Bastola}, the quality-diversity kernel is formed once from initialization-time features and loss summaries. The selected subset is then kept fixed during training. In all these methods, the DPP kernel itself is never updated from fresh training signals. A truly round-adaptive extension of these methods is possible in principle, but it would require repeated collection of fresh client-side information, online kernel reconstruction, and repeated DPP inference during training, which increases overhead and complicates analysis. In contrast, our proposed \textsc{ADCS} explicitly constructs update-dependent kernels in each round of training, and we analyze the resulting gradient approximation bias. Furthermore, unlike the heuristic nature of~\citep{INFOCOM21_Li,ICASSP23_Zhang,DCOSS24_Bastola}, we provide performance guarantee for our approach in terms of a novel FL convergence bound despite the bias in client scheduling.
\section{Preliminaries}\label{sec:DPP:preliminaries}
\subsection{FL System}
We consider a standard FL system that includes a central server and multiple clients, which may be mobile or IoT devices~\citep{MLSys19_Bonawitz}. We denote the index set of clients by $\cN = [N] = \{1,2,\dots,N\}$. The server coordinates the training of a machine learning model utilizing the local datasets of the clients. The optimization problem of model training is given by
\begin{align}\label{eqn:def:finitesum}
    \min_{\model \in \modelset} f(\model)  :=  \sum_{n =1}^{N} \weightatn f_{n}(\model),
\end{align}
where $\model \in \modelset \subseteq \mathbb{R}^{d}$ is a vector containing the model parameters, $\weight \in \Delta_{N-1} \triangleq \{\weight \mid \weightatn \!\geq\! 0, \forall n \!\in\! [N] \text{ and } \sum_{n=1}^{N}\weightatn = 1\}$ is a given weight vector, and $f_{n}(\cdot)$ is the local loss function that is only accessible at client $n$.

In general, an FL algorithm consists of the following five steps in each of $T$ training rounds. i) The central server selects a subset of clients for participation. ii) The central server broadcasts the current global model and sometimes also auxiliary variables to all selected clients. iii) Each selected client performs local updates on the model parameters and the auxiliary variables if needed. iv) Each selected client sends the updated model parameters or the updates, sometimes with the auxiliary variables, to the central server. v) The central server aggregates all local updates to obtain a new global model and updates auxiliary variables if needed. We use $\modelatt$ to represent the global model at the server at the beginning of round $t$. We use $\modelattatn$ to represent the local model at client $n$ before any local update in round $t$ and $\model_{n}^{(t, e)}$ to represent the local model after ${e}$ steps of local update. We denote the local update for client $n$ in round $t$ as $g_{n}^{(t)}$, which can have different forms depending on what local computation is used for local model updates in the FL algorithms. For example, if the local computation is full-batch gradient descent, then $g_{n}^{(t)} = \nabla f_{n}(\modelattatn)$; if the local computation is single step mini-batch gradient descent, then $g_{n}^{(t)} = \nabla f_{n}(\modelattatn; \xi^{(t)})$; if the local computation is $E$ step mini-batch gradient descent, then $g_{n}^{(t)} = {\sum_{e=0}^{E-1}} \nabla f_{n}(\model_{n}^{(t, e)}; \xi^{(t, e)})$. {Here $E$ is the number of local SGD steps.}

In real-world FL applications, it is not recommended or even prohibitive to schedule all clients for training in each round. For example, due to the large number of clients that have heterogeneous computation capabilities and channel conditions, selecting all of them can result in an extremely high round time dominated by the stragglers. Therefore, in this work we focus on partial client scheduling.

If all clients were scheduled for participation, the global update would be $\bar{g}^{(t)}_{\cN} = {\sum_{n=1}^{N}\weightatn g_{n}^{(t)}}$. For partial client participation, the server only selects a subset of client $\cS^{(t)}$ and constructs global update $\bar{g}^{(t)}_{\cS^{(t)}} = {\frac{N}{m}}\sum_{n \in \cS^{(t)}}\weightatn g_{n}^{(t)}$. We determine $\cS^{(t)}$ by first modeling the distribution of all subsets of size $m$ as a DPP, which captures the quality and diversity of the clients, and then finding the mode of that distribution.

\subsection{Diversity Modeling with DPPs}
We follow the general approach to construct a DPP via $L$-ensemble, which directly models the atomic probabilities of all subsets of the ground set $\cN$, i.e., the set of all clients. This definition of $L$-ensemble was first proposed in~\citep{JSP05_Borodin} and is a common choice when modeling real data~\citep{FTML12_Kulesza}. Specifically, the probability of the random subset $\bZ$ realized to be $Z$ is defined as
\begin{align}\label{eqn:L_ensemble}
    P_{L}(\bZ = Z) = \frac{\det(L_{Z})}{\det(L+I_{N})},
\end{align}
where $L$ is a positive semidefinite matrix, $L_{Z}$ is the restriction of $L$ to the set $Z$, and $I_{N}$ is an $N\times N$ identity matrix. 

A key property of DPP is that it promotes diversity. Suppose the representation vector of an item $n$ is $b_{n}$, which is also the $n$th column of a matrix $B$. Then, we can define the kernel $L = B^\top B$. Clearly, $L_{nn'} = \langle b_{n}, b_{n'}\rangle$ can be viewed as a measure of similarity between items $n$ and $n'$, which corresponds to the celebrated cosine similarity if all $b_{n}$'s are normalized. Geometrically, $\det(L_{Z})$ is the squared volume spanned by all vectors in $\{b_{n}: n \in Z\}$. Thus, the subset $Z \subseteq \cN$ whose corresponding $L_{Z}$ spans a larger volume will be assigned a higher probability by definition.

Here, we further restrict our consideration to $m$-DPP, which defines a probability distribution over all subsets $Z \subseteq \cN$ that have cardinality of exactly $m$~\citep{ICML11_Kulesza}.\footnote{Note that in the literature $m$-DPP is often called $k$-DPP. We use $m$ instead of $k$ to avoid notational confusion since we use $k$ to represent the cluster index of clients.} Formally, an $m$-DPP can be generated by conditioning a standard $L$-ensemble DPP on the event that the random set $\bZ$ has cardinality $m$, i.e.,
\begin{align}
    P_{L}^{m}(\bZ = Z) = \frac{\det(L_{Z})}{\sum_{|Z'|=m}\det(L_{Z'})}.
\end{align}

% We consider finding the subset of clients $Z \subseteq \cN$ with cardinality $m$ that has the highest probability, i.e.,
% \begin{align}
%     \argmax_{Z: |Z| = m} \; \log\det(L_{Z}).
% \end{align}
% It is clear that $m$-DPP allows explicit control over the number of selected items.\footnote{We can also formulate an optimization problem with the constraint being $|Z| \leq m$. We choose the equality for simplicity.} 
% This is later used to model client selection in FL, where $m$ clients out of $N$ clients are selected for training. 

% We will explain this in detail in the next subsection.

% However, unlike sampling, marginalization, and conditioning of DPPs, finding the most possible configuration, i.e., maximum a posterior (MAP), in both DPP and $m$-DPP is NP-Hard~\citep{FTML12_Kulesza}. Approximation or greedy algorithms are commonly used for MAP inference in DPP and $m$-DPP~\citep{NeurIPS12_Gillenwater,ICML17_Han,NeurIPS18_Chen,NeurIPS22_Hemmi}.

% based on the submodularity of the logarithm of its probabilities. Informally, a submodular set function has diminishing returns, i.e., adding more input has a decreasing extra benefits. Mathematically, for any $L$-ensemble DPP, $P_{L}$ is log-submodular, i.e.,
% \begin{align}
%     \log(P_{L}(Z' \cup \{n\})) - \log(P_{L}(Z')) \leq  \log(P_{L}(Z \cup \{n\})) - \log(P_{L}(Z)),
% \end{align}
% for any $Z, Z' \subseteq \cZ$ with $Z \subseteq Z'$ and any $n \in \cZ\setminus Z'$. 

\section{Adaptive Determinantal Client Scheduling}\label{DPP:sec:algorithm}
% \subsection{\textsc{ADCS} Algorithm Design}
In the following, we present the proposed \textsc{ADCS} algorithm to effectively optimize a weighted combination of the diversity and quality of the scheduled clients in FL.

\subsection{\texorpdfstring{$L$}{L}-ensemble Kernel Design}
Designing an effective $L$-ensemble kernel for FL is non-trivial. An ideal kernel should capture redundancy among clients under the current global model, while remaining computable at the server without access to raw local data distributions. Since clients with similar data distributions tend to induce similar local updates under any current model, we use local updates as a dynamic proxy for client similarity. Specifically, we use the local update $g_{n}^{(t)} \in \mathbb{R}^{d}$ of each client $n$ to design the $L$-ensemble in round $t$. Let the normalized local update of client $n$ be $\phi_{n}^{(t)} = g_{n}^{(t)}/\|g_{n}^{(t)}\|$ and the update matrix $\Phi^{(t)} = [\phi_{1}^{(t)}, \dots, \phi_{N}^{(t)}]$. We design the kernel as a positive-semidefinite Gram matrix 
\begin{align}\label{eqn:L_ensemble:kernel}
    C^{(t)} = \left(\Phi^{(t)}\right)^\top \Phi^{(t)}.
\end{align} 
Equivalently, the element on any row $n \in [N]$ and any column $n' \in [N]$ of $C^{(t)}$ is defined as $C_{nn'}^{(t)} = (\phi_{n}^{(t)})^\top \phi_{n'}^{(t)}$. The kernel $ C^{(t)}$ reflects directional similarity.

The matrix $C^{(t)}$ is used as the diversity term in the following client scheduling optimization formulation
\begin{align}\label{eqn:DPP:log_prob}
    \max_{\cS^{(t)}: |\cS^{(t)}| = m} \; & (1-\theta)\log\det\left(C_{\cS^{(t)}}^{(t)}\right) + \theta \sum_{n \in \cS^{(t)}} \qualityattatn,
\end{align}
where $C_{\cS^{(t)}}^{(t)}$ is the restriction of $C^{(t)}$ to the set $\cS^{(t)}$, $\qualityattatn \in \mathbb{R}_{+}$ is the quality score of client $n$ in round $t$, and $\theta \in [0, 1]$ is a hyper-parameter that balances the contribution of the quality and the diversity of clients. This formulation lets the kernel encode diversity through $C^{(t)}$ and lets the additive quality term separately measure the individual utility of each client. 

When $\theta = 1$, the maximization problem reduces to the problem of finding the clients of top $m$ largest quality scores. If we choose the quality scores to be the local loss values, we recover the objective of \textsc{PowerOfChoice} in~\citep{AISTATS22_Cho}. When $\theta = 0$, we obtain the most diverse solution based on the $C^{(t)}$ kernel without any consideration of quality. Our formulation allows a flexible trade-off between diversity and quality. Furthermore, the quality scores can have multiple components. For example, if $\qualityattatn = (1-\gamma)\learnqualityattatn + \gamma \commqualityattatn$ with the additional $\gamma \in [0, 1]$, our formulation can trade off between communication-related and learning-related quality terms.

\subsection{MAP Inference of \texorpdfstring{$L$}{L}-ensemble DPP}\label{sec:MAP_DPP}
In this section, we discuss the approach to solve~\eqref{eqn:DPP:log_prob}. Applying standard DPP quality-diversity decompositions~\citep{FTML12_Kulesza}, the objective in~\eqref{eqn:DPP:log_prob} can be equivalently written as another DPP when $\theta \in [0, 1)$. The corresponding $L$-ensemble kernel is
\begin{align}\label{dpp:eqn:equi_dpp}
    L^{(t)} = Q^{(t)} C^{(t)} Q^{(t)},
\end{align}
where $Q^{(t)} \triangleq \diag(\exp(\alpha \qualityatt))$, $\qualityatt \triangleq [\qualityattatone, \dots, \qualityattatN]$ is the vector of quality scores, and $\alpha \triangleq \frac{\theta}{2(1-\theta)} \in [0, \infty)$.

In each round $t$, given our designed kernel $L^{(t)}$, we solve the following MAP inference to determine the most possible configuration of a subset of clients with cardinality $m$ 
\begin{align}\label{eqn:dpp:map}
    \max_{\cS^{(t)}: |\cS^{(t)}| = m} \; \log\det\left(L^{(t)}_{\cS^{(t)}}\right).
\end{align}
It is well-known that performing MAP inference of an $m$-DPP is NP-hard~\cite[Section 5.2.5]{FTML12_Kulesza}, but efficient heuristics are available in the literature~\citep{NeurIPS18_Chen,NeurIPS12_Gillenwater,ICML17_Han,NeurIPS22_Hemmi}. We choose the fast greedy MAP inference algorithm from~\citep{NeurIPS18_Chen} in our implementation. $\cS^{(t)}$ is initialized with the item $i$ that has the largest value of $L^{(t)}_{i,i}$. Then $\cS^{(t)}$ is updated until $m$ items are selected by adding an item $j$ in each iteration where 
\begin{align}
    j = \argmax_{i \in \cN\setminus\cS^{(t)}} \;  \log\det\left(L^{(t)}_{\cS^{(t)} \cup \{i\}}\right) - \log\det\left(L^{(t)}_{\cS^{(t)}}\right),
\end{align}
i.e., $j$ has the largest gain of the log determinant for the new submatrix among all unselected items. Furthermore, by using Cholesky decomposition, we can avoid the brute-force calculation of the values of $\det(L^{(t)}_{\cS^{(t)} \cup \{i\}})$ in each iteration and achieve time complexity of $\cO(m^{2}N)$~\citep{NeurIPS18_Chen}.

% The marginal gain of log-probability can be determined by
% \begin{align}
%     \log P(\cS^{(t)} \cup \{i\}) - \log P(\cS^{(t)}) \sim (1-\theta)(\log\det(C_{\cS^{(t)}}^{(t)} \cup \{i\}) - \log\det(C_{\cS^{(t)}}^{(t)})) + \theta \qualityattati.
% \end{align}

% Note that when $m=1$, the optimization problem reduced to find the client with the largest squared norm of its local updates $\|g_{n}^{(t)}\|^{2}$. However, when $m > 1$, the MAP inference from $m$-DPPs will not na\"ively output the clients with top $m$ largest squared norm of its local updates but will take into account both diversity and quality.
\subsection{Overall~\textsc{ADCS} Algorithm}
The $L$-ensemble kernel construction step described above requires the local gradients from all clients. In practical implementation, this step does not need to be performed in every round, since the gradient directions do not change drastically over consecutive rounds. We use a tunable parameter $R$ to specify the number of rounds for each update of $L^{(t)}$.\footnote{Compressed approximations of $g_{n}^{(t)}$ may also be applied to further reduce the communication overhead while preserving the accuracy of estimating their inner products (e.g.,~\citep{KDD06_Li}). However, that is outside the scope of this paper.} Then, in each training round t, the server applies the MAP inference algorithm on the current kernel matrix $L^{(t)}$ to obtain the scheduled clients $\cS^{(t)}$. Finally we perform model broadcasting, local training, and aggregation over local updates. The pseudocode of the overall algorithm is shown in Algorithm~\ref{alg:ADCS}. 

The total extra computational complexity of \textsc{ADCS} on the server is $\cO((N^{2}d+m^{2}N)T/R)$, which includes both the construction and MAP inference of the kernel. Since the construction and MAP inference of the kernel are both performed at the server, which often has sufficient computational capabilities, the extra computation time for \textsc{ADCS} is negligible. The exact client-side computation is $mT+(N-m)\lceil T/R\rceil$ updates, where $mT$ is the usual computation overhead and $(N-m)\lceil T/R\rceil$ is the extra computation for ADCS. With $R=20,N=30,m=18$ in our experiments, the extra client-side computation overhead for ADCS is only $3.3\%$, which is negligible.

\begin{algorithm}[t]
\renewcommand{\algorithmicrequire}{\textbf{Input:}}
\renewcommand{\algorithmicensure}{\textbf{Output:}}
\caption{Adaptive Determinantal Client Scheduling (\textsc{ADCS})}\label{alg:ADCS}
\begin{algorithmic}[1]
\REQUIRE initial model $\model^{(0)}$, hyper-parameter $\theta$, weight vector $\weight$, server learning rate $\eta_{s}$, total number of rounds $T$, and number of selected clients $m$.
\ENSURE $\hat{\model} = \{\modelatt\}_{t=1}^{T}$.
% \STATE Server broadcasts $\model^{(0)}$ to all the $N$ clients.
% \FOR{each client $n \in[N]$}
%     \STATE  Client $n$ computes local update $g_{n}^{(0)}$.
%     \STATE  Client $n$ send local update $g_{n}^{(0)}$ to the server.
% \ENDFOR
% \STATE Server constructs the kernel matrix $L^{(1)}$.
\FOR{each round $t = 0, \dots, T-1$}
    % \STATE Server solves~\eqref{eqn:dpp:map} to get $\cS^{(t)}$ using the kernel matrix $L^{(1)}$.
    \IF{$t \bmod R = 0$}
        \STATE Server constructs a new kernel matrix $L^{(t)}$.
    \ELSE
        \STATE The kernel matrix stays unchanged $L^{(t)} = L^{(t-1)}$.
    \ENDIF
    \STATE Server solves the maximization problem of~\eqref{eqn:DPP:log_prob} to get $\cS^{(t)}$.
    \STATE Server selects clients $\cS^{(t)} \subseteq \cN$.
    % \STATE Server selects clients $\cS^{(t)} \subseteq \cN$ by solving a $m$-DPP.
    \STATE Server broadcasts $\modelatt$ to $\cS^{(t)}$.
    \FOR{each client $n \in \cS^{(t)}$}
        \STATE  Client $n$ computes local update $g_{n}^{(t)}$.
        \STATE  Client $n$ sends local update $g_{n}^{(t)}$ to the server.
    \ENDFOR
    \STATE Server constructs global update $\bar{g}^{(t)}_{\cS^{(t)}} = \frac{N}{m}\sum_{n \in \cS^{(t)}}\weightatn g_{n}^{(t)}$.
    \STATE Server updates the global model via $\modelattplusone = \modelatt - \eta_{s} \bar{g}^{(t)}_{\cS^{(t)}}$.  
    % \STATE Server updates the kernel matrix $L^{(t+1)}$.
    % \STATE Server updates the kernel matrix $L^{(t+1)}$ via~\eqref{eqn:L:update}.
\ENDFOR   
\end{algorithmic}
\end{algorithm}

\section{FL Convergence under \textsc{ADCS}}
\label{sec:nonconvex_dpp_fedavg}
Since client scheduling with diversity consideration typically leads to biased aggregated gradients, i.e., $\mathbb{E}[\bar{g}^{(t)}_{\cS^{(t)}}] \neq \bar{g}^{(t)}_{\cN}$, the standard proof of convergence analysis of FL algorithms with unbiased gradient approximation does not apply~\citep{ICLR19_Stich, ICLR20B_Li}. In the following, we first quantify the gradient approximation error $\|\bar{g}^{(t)}_{\cS^{(t)}} - \bar{g}^{(t)}_{\cN}\|$ for \textsc{ADCS}, and then provide a novel convergence bound for general biased client selection in FL for non-convex loss functions. For improved readability of our derivation, we assume a uniform weight vector $\weight$. Our analysis can be directly extended to general $\weight$ by replacing all $1/N$, $1/m$, and count-based cluster terms by weighted population, selection, and cluster-mass terms.

\subsection{Gradient Approximation Bias of \textsc{ADCS}}
\label{DPP:section:gradient_approximation_error}
For this theoretical analysis, we consider the idealized scenario of $R=1$. However, in Section~\ref{DPP:sec:experiments}, we will show that \textsc{ADCS} retains its performance advantage for a wide range of $R$ values. We omit the time index $t$ since the following analysis holds for any $t$. Let $g_1,\dots,g_N \in \mathbb{R}^d$ be the client gradients. We further define $G=[g_1,\dots,g_N]\in\mathbb{R}^{d\times N}$ and $L=G^\top G$ such that $L_{ij}=\langle g_i,g_j\rangle$. Let $S^\star$ be the selected subset, of size $m$, from any client selection algorithm. Let $\bs\in\{0,1\}^N$ denote the indicator vector of $S^\star$, so that $\mathbf{1}^\top \bs=m$, where $\mathbf{1}\in\mathbb{R}^N$ is the all-ones vector. Define the all-client average by $\bar g_N=\frac1N\sum_{i=1}^N g_i=\frac1N G\mathbf{1}$ and the selected average by $\bar g_S=\frac1m\sum_{i\in S^\star} g_i=\frac1m G\bs$.

We consider a general cluster structure for the client gradients. Suppose $g_1,\dots,g_N$ are partitioned into $K$ clusters $C_1, ..., C_K$, i.e., $[N]=\cup_{k=1}^K  C_k$, and $n_k =|C_k|$. Let $z_i$ denote the cluster assignment of client $i$, such that $z_i=k$ if and only if $i\in C_k$. Let $c_k$ be a representation vector for cluster $C_k$. In this analysis, we allow any general $c_k$, e.g., it may be the centroid of $C_k$. Let $\kappa \triangleq \min_{1\le k\le K}\|c_k\|_2$. We define
\begin{align}
r \triangleq \max_{i\in[N]} \| g_i - c_{z_i} \|_2, \label{eqn:intra-cluster-concentration}
\end{align}
and
\begin{align}
\mu \triangleq \max_{k\neq \ell} \left|
\frac{c_k^\top c_\ell}{\|c_k\|_2\|c_\ell\|_2}
\right|.  \label{eqn:inter-cluster-separation}
\end{align}
Thus, $r$ represents the intra-cluster concentration, and $\mu$ represents the inter-cluster separation. We further define quality-related terms $w_i \triangleq e^{2\alpha q_i}$, $\bar w \triangleq \max_{1\le i\le N} w_i$, $\bar w_k \triangleq \max_{i\in \mathcal C_k} w_i$, and $Q_m
\triangleq
\max_{\substack{T\subseteq[K]\\ |T|=m}}
\prod_{k\in T}\bar w_k$.

\begin{assumption}[Cluster structure of gradients]\label{dpp:assumption:cluster}
We assume
\begin{align}
\mu < \frac{1}{m-1}, 
% \mu < \frac{1}{N}, 
\label{eqn:assumption:inter-cluster-separation}
\end{align}
i.e., the clusters are well separated, and
\begin{align}
 r  < \min \left(\kappa, r_\star, \frac{\kappa}{2}\sqrt{\frac{1-(m-1)\mu}{m}}\right), 
 \label{eqn:assumption:intra-cluster-concentration}
\end{align}
where $r_\star$ is the smallest $r \geq 0$ such that $Q_m \rho^{2m}-16r^2\bar w^{\,m}/\kappa^2 = 0$ and $\rho \triangleq
\sqrt{1-(m-1)\mu} - 2r\sqrt{m}/\kappa > 0$. Condition~\eqref{eqn:assumption:intra-cluster-concentration} ensures sufficient intra-cluster concentration.
\end{assumption}

Let $S_{\text{ADCS}}^\star$ be the set of clients chosen by {exact MAP for} \textsc{ADCS}. Define the aggregated gradients $\bar g_{S_{\text{ADCS}}^\star}\triangleq\frac1m\sum_{i\in S_{\text{ADCS}}^\star} g_i$ and selected cluster subset $T_{\text{ADCS}}^\star \triangleq \{z_i:i\in S_{\text{ADCS}}^\star\}$.

\begin{theorem}[\textsc{ADCS} bias bound]
\label{dpp:thm:cluster_adaDPP}
With Assumption~\ref{dpp:assumption:cluster} and further assuming $m \leq K$, we have
\begin{equation}
\left\|
\bar g_{S_{\text{ADCS}}^\star}-\bar g_N
\right\|_2
\le
\left\|
\sum_{k=1}^K
\left(
\frac{\mathbf{1}\{k\in T_{\text{ADCS}}^\star\}}{m}
-
\frac{n_k}{N}
\right)c_k
\right\|_2
+2r.
\label{dpp:eqn:thm1results}
\end{equation}
\end{theorem}
\begin{proof}
The key idea is to show that every exact MAP solution $S_{\text{ADCS}}^\star$ contains clients from $m$ distinct clusters, i.e.
$\big|\{z_i:i\in S_{\text{ADCS}}^\star\}\big|=m$. The details are deferred to Appendix~\ref{dpp:appendix:thm1}. 
\qedhere
\end{proof}

We remark that for any subset of clients of size $m$, the discrepancy between the selected average and the full average is upper bounded as $\|\bar g_S-\bar g_N\|_2
\le
\sqrt{\lambda_{\max}(L)\left(\frac1m-\frac1N\right)}$,
where $\lambda_{\max}(L) \geq 0$ is the largest eigenvalue of $L$. This can be shown by letting $\bv=\frac{1}{m}\bs-\frac{1}{N}\mathbf{1}$ and observing that
\begin{equation}
\|\bar g_S-\bar g_N\|_2^2=\bv^\top L\bv
\leq \lambda_{\max}(L)\|\bv\|_2^2
= \lambda_{\max}(L)\left(\frac1m-\frac1N\right),
\end{equation}
where the inequality is from the Rayleigh quotient inequality. In the following corollary, we show that the bound in Theorem~\ref{dpp:thm:cluster_adaDPP} is tighter than the general bound under the following strong cluster setting: (1) $r=0$, (2) \(c_1,\dots,c_K\in\mathbb{R}^d\) are orthonormal, i.e., $c_k^\top c_\ell = \mathbf{1}\{k=\ell\}$, and (3) the clusters are balanced, i.e., $n_k = \frac{N}{K}$, for all $k\in[K]$. The proof is given in Appendix~\ref{dpp:appendix:cor1}.

\begin{corollary}
\label{dpp:corollary:comparison}
Under the strong cluster setting, we have
\begin{align}
\|\bar g_{S_{\text{ADCS}}^\star} - \bar g_N\|_2
\le
\sqrt{\frac{K-m}{Km}}.
\end{align}
\end{corollary}

In the same cluster setting, we can also further quantify the general bound. Since $G=[g_1,\dots,g_N]$ and $g_i=c_{z_i}$, we have $GG^\top = \sum_{i=1}^N g_i g_i^\top = \sum_{k=1}^K n_k\, c_k c_k^\top = \frac{N}{K}\sum_{k=1}^K c_k c_k^\top$. Because \(c_1,\dots,c_K\) are orthonormal, \(\sum_{k=1}^K c_k c_k^\top\) is the orthogonal projector onto the span of $c_1,\dots,c_K$, whose largest eigenvalue is \(1\). We have $\lambda_{\max}(GG^\top)=N/K$. Since \(L=G^\top G\) and \(GG^\top\) have the same nonzero eigenvalues, $\lambda_{\max}(L)=N/K$. The general bound above becomes 
\begin{equation}
\|\bar g_{S_{\text{ADCS}}^\star} - \bar g_N\|_2
\le
\sqrt{\lambda_{\max}(L)\left(\frac{1}{m}-\frac{1}{N}\right)}
=
\sqrt{\frac{N-m}{Km}}.
\end{equation}
Since $N > K$ in general, the bound for \textsc{ADCS} is strictly tighter.

\subsection{FL Convergence Bound}
Since adding diversity consideration to client scheduling can lead to bias as shown in~\eqref{dpp:eqn:thm1results}, we provide a novel convergence analysis for client scheduling with selection bias in this section. Let $\mathcal{F}_t$ be the sigma-field generated by all randomness up to and including the start of round $t$. Let $f^\star\triangleq\inf_{\model} f(\model)>-\infty$. We make the following assumptions, which are common in the literature on FL analysis.
\begin{assumption}[Smoothness]\label{DPP:assumption:smoothness}
Each $f_i$ is $L$-smooth, i.e., there exists a positive $L$ such that 
\begin{align}
\|\nabla f_{i}(\model_1) - \nabla f_{i}(\model_2) \| \leq L\|\model_1- \model_2\|,
\end{align}
holds $\forall \model_{1}, \model_{2} \in \modelset$ and $\forall i \in [N]$. Hence $f$ is also $L$-smooth.
\end{assumption}

\begin{assumption}[Unbiased Local Stochastic Gradients]\label{DPP:assumption:unbiasedness}
For all $i,t,e$,
\begin{align}
\mathbb{E}[g^{(t,e)}_{i}\mid \model^{(t,e)}_{i}]
= \nabla f_i(\model^{(t,e)}_{i}).
\end{align}
\end{assumption}

\begin{assumption}[Bounded Second Moment]\label{DPP:assumption:bounded2moment}
There exists $G>0$ such that for all $i,t,e$,
\begin{align}
\mathbb{E}[\|g^{(t,e)}_{i}\|^2\mid \mathcal{F}_t]\le G^2.
\end{align}
\end{assumption}

\begin{assumption}[Bounded Per-client Variance]\label{DPP:assumption:variance}
There exists $\sigma^2\ge 0$ such that for all $i,t,e$, \begin{align}
\mathbb{E}[\|g^{(t,e)}_{i}-{\nabla f_i}(\model^{(t,e)}_{i})\|^2 \mid \mathcal{F}_t]\le \sigma^2.
\end{align}
\end{assumption}

Define the expected client selection bias $b_t \triangleq \mathbb{E}[\bar g_t\mid \mathcal{F}_t] - \nabla f(\modelatt)$. Let $\epsilon$ be an upper bound on $\|b_t\|$ for all $t$. When the local gradient update uses the entire local dataset, $\nabla f(\modelatt)$ is the same as $\bar g_N$ in round $t$ in the LHS of~\eqref{dpp:eqn:thm1results}. Since for \textsc{ADCS}, $\mathbb{E}[\bar g_t\mid \mathcal{F}_t] = \bar g_t$, $\epsilon$ can be set as the RHS of~\eqref{dpp:eqn:thm1results} if the cluster assumption holds. Define the zero-mean selection deviation $\zeta_t \triangleq \bar g_t - \mathbb{E}[\bar g_t\mid \mathcal{F}_t]$. Let $\nu_{\rm sel}$ be an upper bound on the selection randomness variance $\mathbb{E}[\|\zeta_t\|^2\mid \mathcal{F}_t]$ for all $t$. We have the following theorem that captures biased client selection, local client drift, gradient approximation error, and stochastic gradients.

\begin{theorem}
\label{thm:nonconvex_dpp_fedavg}
Suppose Assumptions~\ref{DPP:assumption:smoothness}-\ref{DPP:assumption:variance} hold. {Let $\eta_\ell>0$ and $\eta_s>0$ denote the local SGD and server step sizes, respectively. Define $\eta\triangleq E\eta_s$ and suppose $\eta\le \frac{1}{8L}$.} Then for any $T\ge 1$,
\begin{align}
&\frac{1}{T}\sum_{t=0}^{T-1}\mathbb{E}\left[\|\nabla f(\modelatt)\|^2\right] \nonumber \\
& \le
\underbrace{\frac{4\big(f(\model_0)-f^\star\big)}{\eta T}}_{\text{optimization error}}
+ \underbrace{\frac{5}{6}L^2\eta_\ell^2 G^2 (E-1)(2E-1)}_{\text{client drift}}
\nonumber \\
& 
\quad +\underbrace{16L\eta\left(\nu_{\rm sel} + \frac{\sigma^2}{mE}\right)}_{\text{variance terms}} + \underbrace{5\varepsilon^2}_{\text{gradient appro. error}}
.
\label{eq:final_bound}
\end{align}
\end{theorem}

Our convergence analysis decomposes the aggregated update $\bar g_t$ into a population gradient term plus two selection-induced effects. Detailed proof is deferred to Appendix~\ref{dpp:appendix:thm2}. 
% $\bar g_t = \nabla f(w_t) + b_t + \zeta_t$. 

{Theorem~\ref{thm:nonconvex_dpp_fedavg} bounds the average squared gradient norm in terms of optimization error, selection bias, local client drift, and stochastic variance.} In particular, the \textsc{ADCS} bias is bounded by Section~\ref{DPP:section:gradient_approximation_error}, and for deterministic MAP selection we have $\nu_{\mathrm{sel}}=0$. {For fixed $E$ and bounds $G$, $\sigma^2$, $\nu_{\rm sel}$, and $\varepsilon$ uniform in $T$, choosing $\eta=E\eta_s=\Theta(T^{-1/2})$ subject to the stepsize condition and $\eta_\ell=\cO(T^{-1/4})$ yields an upper bound of $5\varepsilon^2+\cO(T^{-1/2})$ in~\eqref{eq:final_bound}.}

\section{Experiments}\label{DPP:sec:experiments}
% Keep each requested page pair inside one indivisible float.
% RawFloats preserves the separate figure and table captions in source order.
\begin{figure}[t]
\RawFloats
\centering
\begin{minipage}{\textwidth}
\captionsetup{type=figure}
\centering
\begin{subfigure}[t]{0.47\textwidth}
    \centering
    \includegraphics[width=\linewidth]{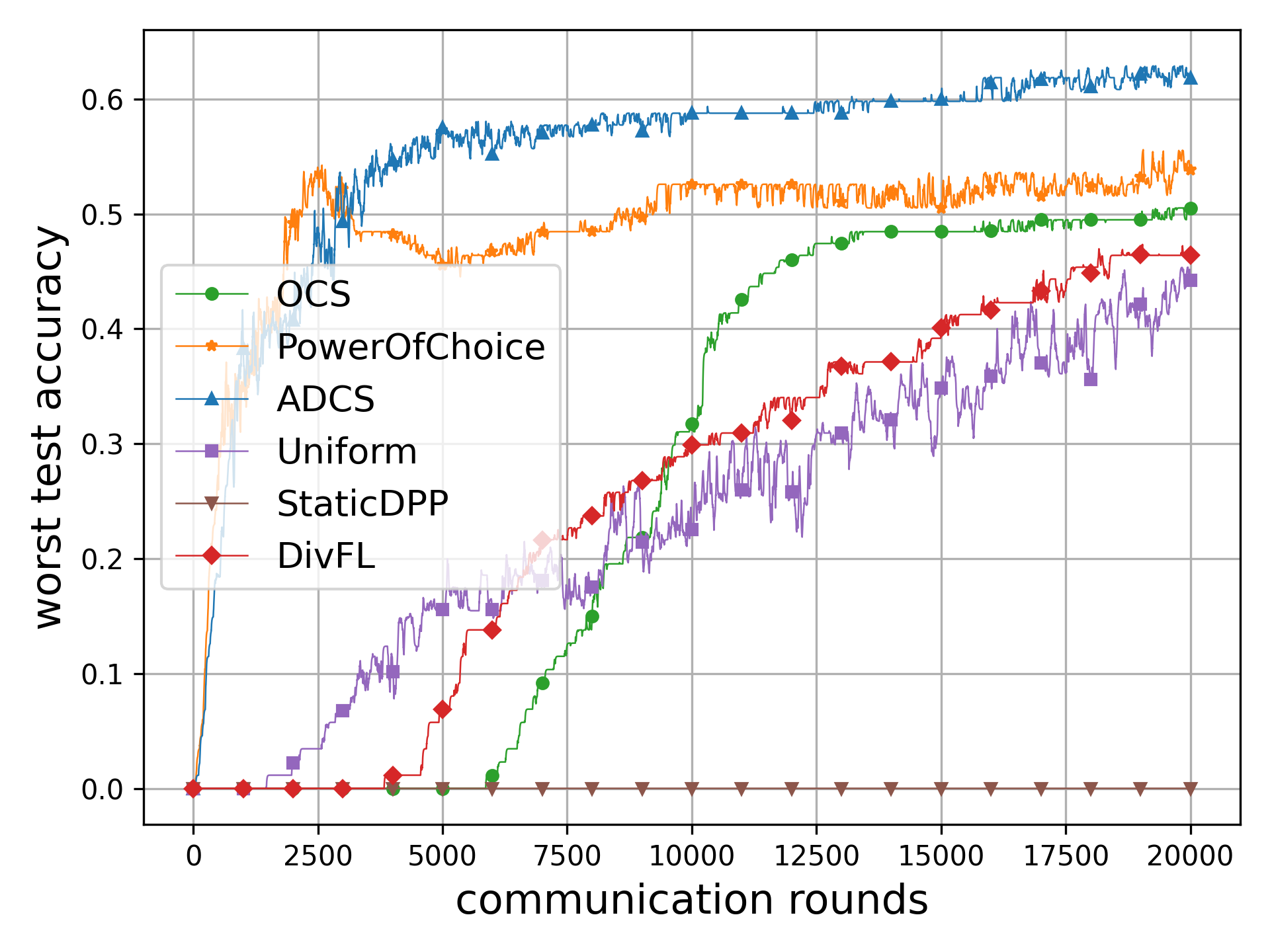}
    \caption{Worst test accuracy}
    \label{dpp:figure:comparison:worst}
\end{subfigure}
\hfill
\begin{subfigure}[t]{0.47\textwidth}
    \centering
    \includegraphics[width=\linewidth]{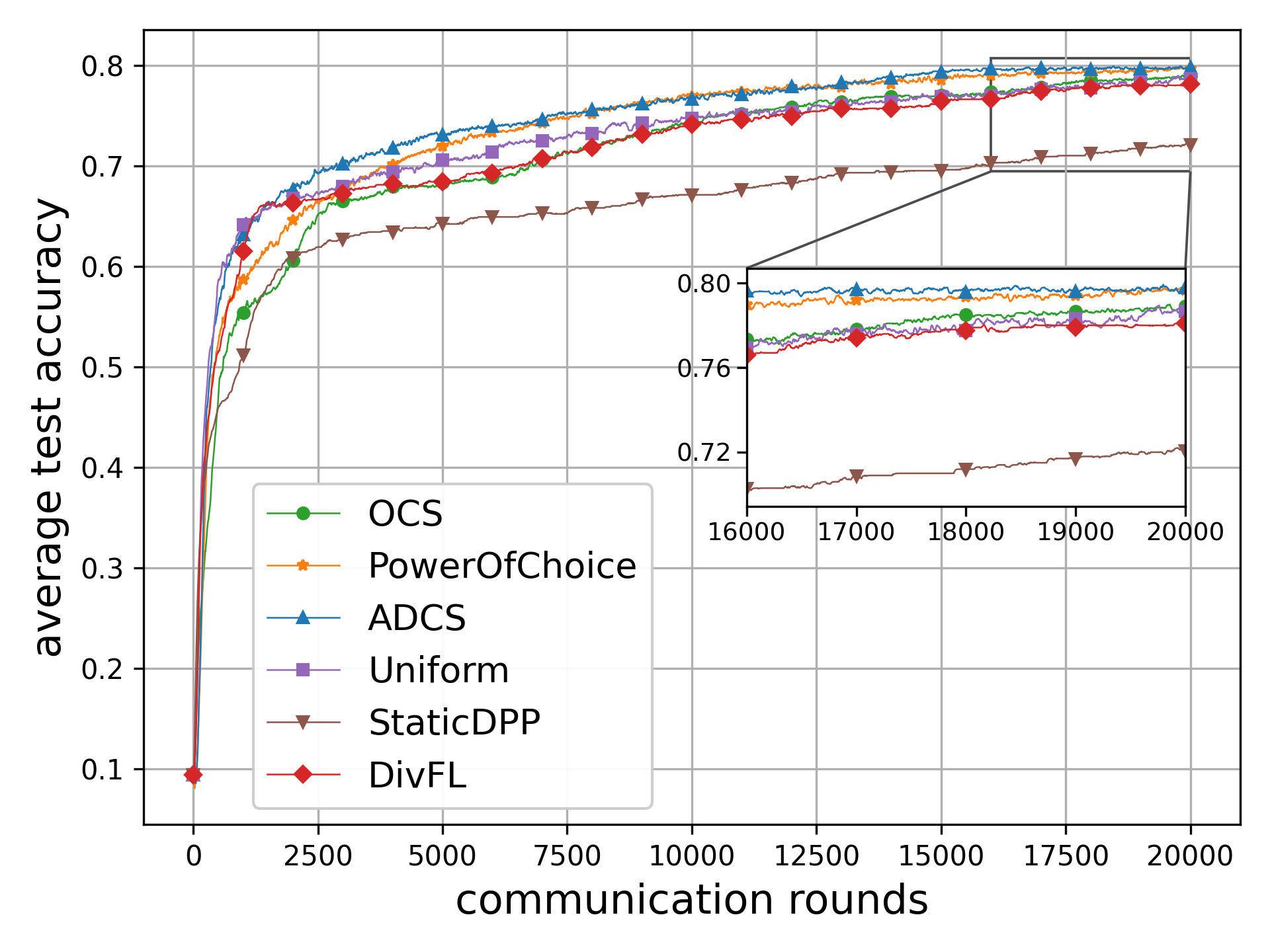}
    \caption{Average test accuracy}
    \label{dpp:figure:comparison:average}
\end{subfigure}
\caption{Comparison of worst and average test accuracies.}
\label{dpp:figure:comparison}
\end{minipage}
\par\vspace{12pt}
\begin{minipage}{\textwidth}
\captionsetup{type=table,position=top}
  \caption{Ratio of Each Group of Clients Selected}
  \label{tab:selection_ratio}
  \centering
  \small
  \setlength{\tabcolsep}{3pt}
  \renewcommand{\arraystretch}{1.15}
  \begin{tabular*}{\textwidth}{@{\extracolsep{\fill}}l*{10}{c}@{}}
    \toprule
    & \multicolumn{10}{c}{Client group} \\
    \cmidrule(lr){2-11}
    Method & $0$ & $1$ & $2$ & $3$ & $4$ & $5$ & $6$ & $7$ & $8$ & $9$ \\
    \midrule
    \textsc{Uniform} & $0.1006$ & $0.0998$ & $0.0997$ & $0.1000$ & $0.0994$ & $0.1006$ & $0.0993$ & $0.1006$ & $0.1001$ & $0.0999$ \\
    \textsc{OCS} & $0.1257$ & $0.0433$ & $0.1661$ & $0.1004$ & $0.1666$ & $0.0583$ & $0.1653$ & $0.0477$ & $0.0733$ & $0.0533$ \\
    \textsc{PowerOfChoice} & $0.1042$ & $0.0441$ & $0.1444$ & $0.0881$ & $0.1388$ & $0.1417$ & $0.1643$ & $0.0644$ & $0.0641$ & $0.0460$ \\
    \textsc{ADCS} & $0.0952$ & $0.0636$ & $0.1127$ & $0.0843$ & $0.1032$ & $0.1641$ & $0.1564$ & $0.0775$ & $0.0821$ & $0.0609$ \\
    % DPP-Norm &
    % $0.1186$ & $0.0460$ & $0.1656$ & $0.0949$ & $0.1665$ & $0.0627$ & $0.1651$ & $0.0507$ & $0.0741$ & $0.0558$ \\
    % \midrule
    \bottomrule
  \end{tabular*}
  \par
\end{minipage}
\end{figure}

\begin{figure}[t]
\RawFloats
\centering
\begin{minipage}{\textwidth}
\captionsetup{type=figure}
\centering
\begin{subfigure}[t]{0.47\textwidth}
    \centering
    \includegraphics[width=\linewidth]{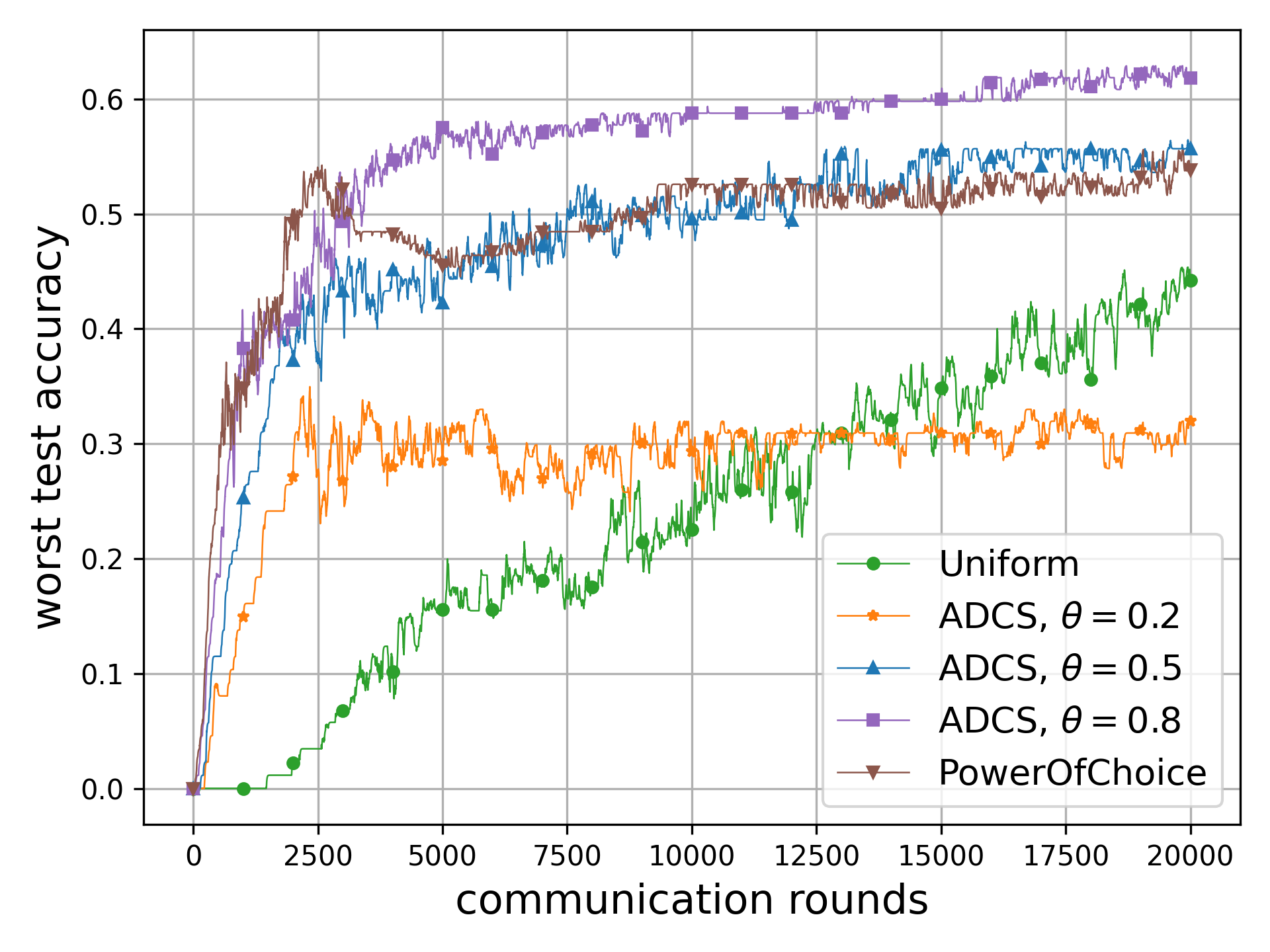}
    \caption{Worst test accuracy}
    \label{dpp:figure:comparison_theta:worst}
\end{subfigure}
\hfill
\begin{subfigure}[t]{0.47\textwidth}
    \centering
    \includegraphics[width=\linewidth]{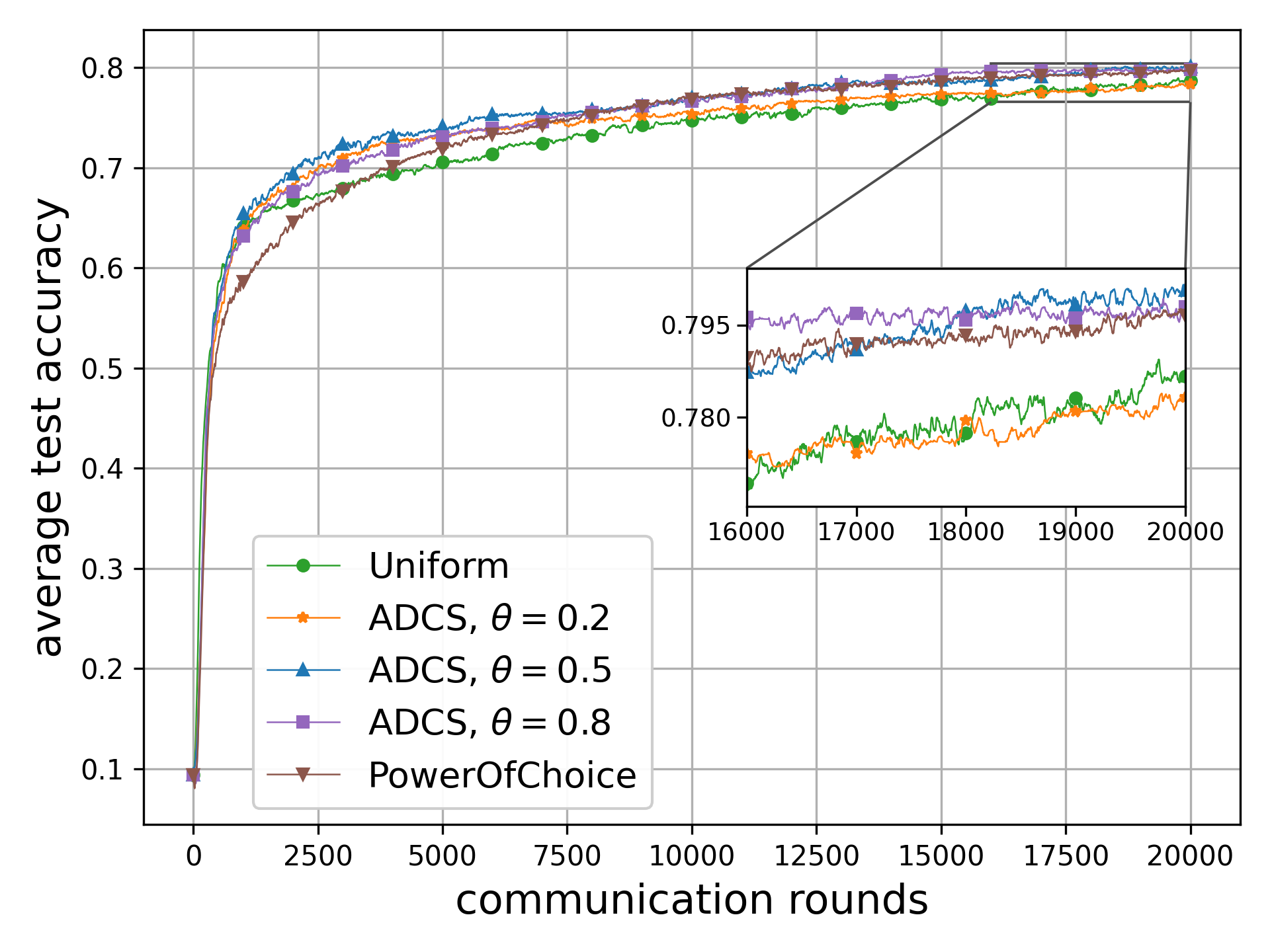}
    \caption{Average test accuracy}
    \label{dpp:figure:comparison_theta:average}
\end{subfigure}
\caption{Worst and average test accuracies for varying $\theta$.}
\label{dpp:figure:comparison_theta}
\end{minipage}
\par\vspace{12pt}
\begin{minipage}{\textwidth}
\captionsetup{type=figure}
\centering
\begin{subfigure}[t]{0.47\textwidth}
    \centering
    \includegraphics[width=\linewidth]{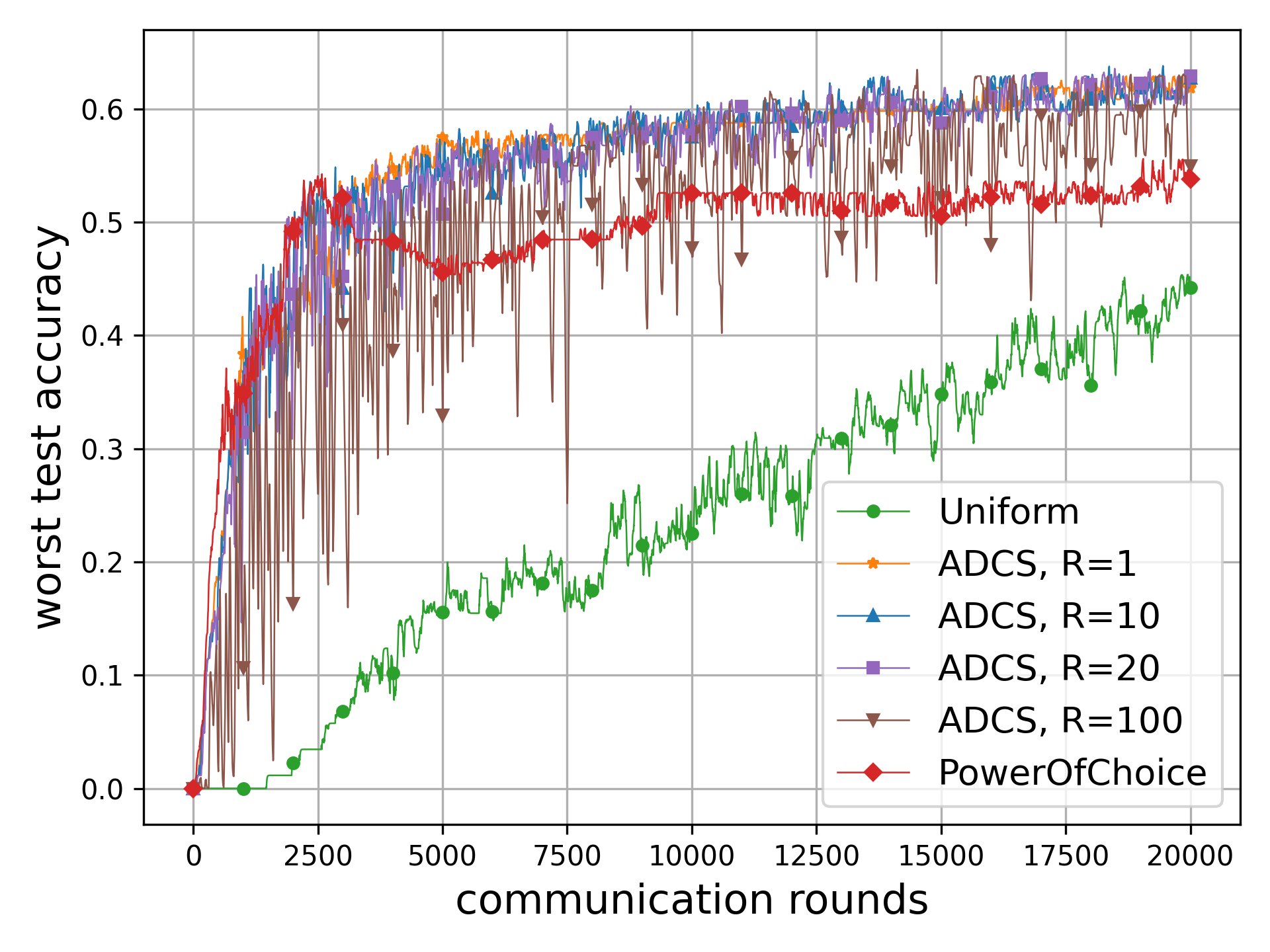}
    \caption{Worst test accuracy}
    \label{dpp:figure:comparison_R:worst}
\end{subfigure}
\hfill
\begin{subfigure}[t]{0.47\textwidth}
    \centering
    \includegraphics[width=\linewidth]{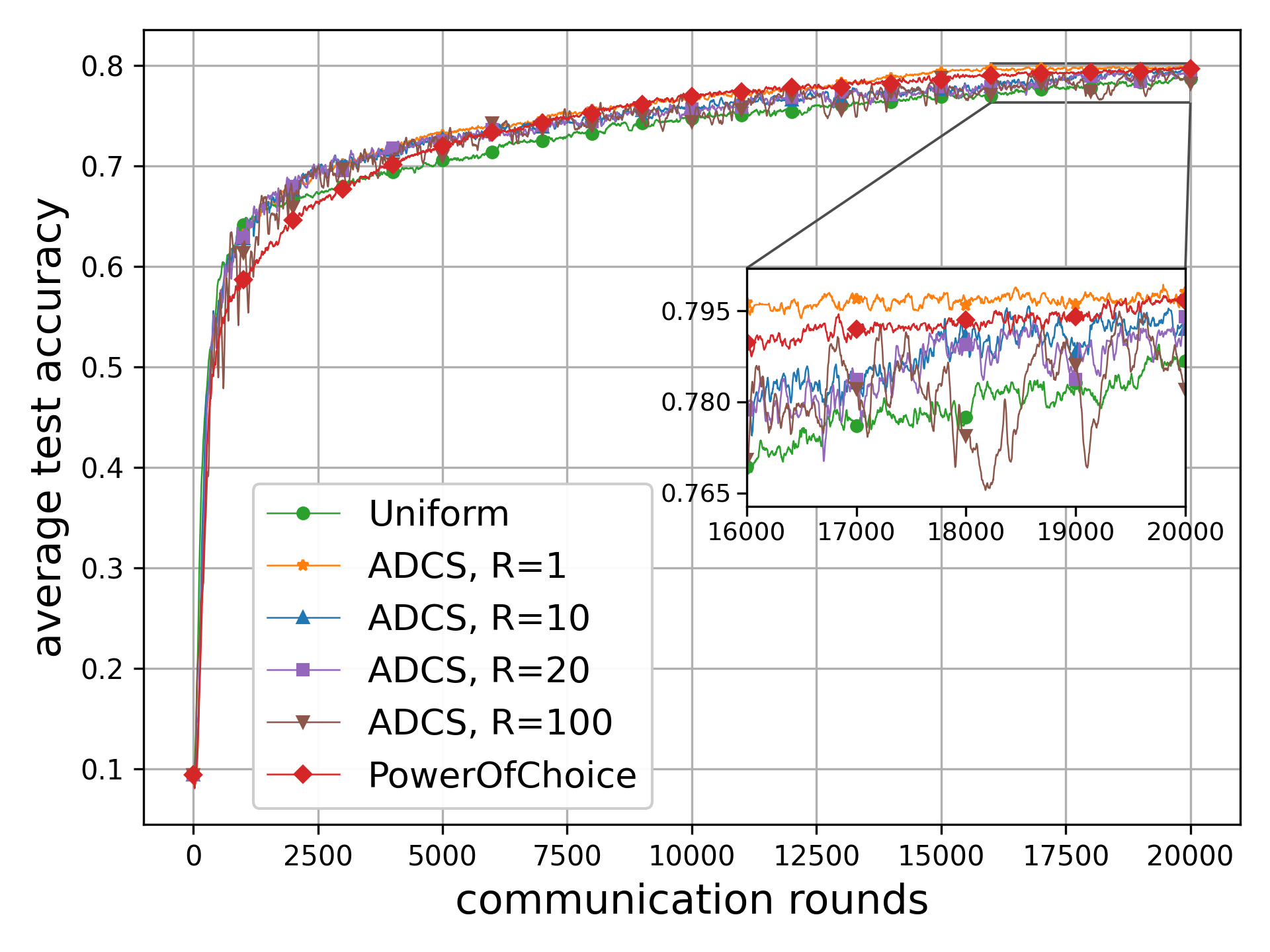}
    \caption{Average test accuracy}
    \label{dpp:figure:comparison_R:average}
\end{subfigure}
\caption{Worst and average test accuracies for varying $R$.}
\label{dpp:figure:comparison_R}
\end{minipage}
\end{figure}

We perform experiments on Fashion-MNIST~\citep{Fashion_MNIST} using multinomial logistic regression {on its official training and test split}. We set $N=30$. Specifically, we have $10$ virtual groups of clients, where each group contains $3$ clients and the underlying data distribution is the same for all clients in the same group. We consider the following benchmarks
\begin{itemize}
    \item \textsc{Uniform}~\citep{AISTATS17_McMahan}: choosing $m$ clients uniformly at random out of $N$ clients in each round.
    % \item \textsc{Unifrom-Probablistic}~\citep{AISTATS17_McMahan}: each client has a sampling probability of $\frac{m}{N}$ and the updates are compensated by dividing $\frac{m}{N}$.
    \item \textsc{PowerOfChoice}~\citep{AISTATS22_Cho}: clients with top $m$ largest local losses.
    \item \textsc{OCS}~\citep{arXiv21_CHR}: clients with top $m$ largest gradient norms (modified from probabilistic setting).
    % \item \textsc{DivFL}~\citep{ICLR22_Balakrishnan}: solving submodular maximization problems.
    % \item \citep{INFOCOM21_Li}
    % \item \citep{ICASSP23_Zhang}
    % \item \citep{DCOSS24_Bastola}
    % \item \textsc{AdaDPP-Norm}: Use the norm of gradients as the quality scores.
    \item \textsc{DivFL}~\citep{ICLR22_Balakrishnan}: This method selects a diverse subset via submodular function maximization.
    \item \textsc{StaticDPP}: This method constructs DPP only once before the training and selects the most diverse subset, which is adapted from previous work~\citep{INFOCOM21_Li,ICASSP23_Zhang,DCOSS24_Bastola}.
    \item \textsc{ADCS}: our proposed method.
\end{itemize}
We perform $T=20000$ rounds of training. For local computation, we use one step of SGD of batch size $16$ with a learning rate of $0.001$ for all methods. Every method schedules $m=18$ clients for partial client participation in each training round.

% \subsection{Synthetic Dataset}
% See the Appendix.
% \subsection{Fashion-MNIST}

\subsection{Performance Comparison}
In Figure~\ref{dpp:figure:comparison}, we compare the test accuracy of \textsc{ADCS} with all benchmarks in terms of training rounds. We show both the worst accuracy, which refers to the worst-case test accuracy among all clients, and the average accuracy, which refers to the average test accuracy among all clients. We set $\theta = 0.8$ for \textsc{ADCS}.  We observe that pure quality-based methods \textsc{OCS} and \textsc{PowerOfChoice}, as well as diversity-based \textsc{DivFL} methods, can achieve better worst test accuracy compared with the na\"ive \textsc{Uniform} algorithm. However, by considering both diversity and quality, \textsc{ADCS} provides the highest
worst test accuracy while maintaining competitive average test accuracy. \textsc{StaticDPP} has significantly low worst test accuracy since the data distributions are highly heterogeneous and only a subset of clients is chosen, leaving many clients untouched.

\subsection{Effects of \texorpdfstring{$\theta$}{theta} and \texorpdfstring{$R$}{R}}
The performance of our \textsc{ADCS} is dependent on the hyper-parameter $\theta$, which trades off between the diversity and quality terms. The results of varying $\theta$ are shown in Figure~\ref{dpp:figure:comparison_theta}. We observe that in the current setting, with a properly chosen $\theta$, both the worst test accuracy and the average test accuracy can be improved over those of the benchmarks. However, with a poorly chosen $\theta$, e.g., $\theta = 0.2$, the worst test accuracy can be even worse than that of \textsc{Uniform}.

The performance of \textsc{ADCS} is also dependent on the hyper-parameter $R$, which determines the frequency of the updates of the DPP kernel $L$. We fix $\theta=0.8$ and vary the update gap $R$. The results are shown in Figure~\ref{dpp:figure:comparison_R}. With a small $R$ up to $20$, the performance of \textsc{ADCS} remains almost identical to that of $R=1$. However, when $R$ becomes too large, the performance of \textsc{ADCS} on the worst-off client can deteriorate and fluctuate.

% \subsection{Comparison of Different Methods with Communication Time}
% We consider heterogeneous communication time between different clients and the server. Note that in the following figures, the $x$-axis is now communication time instead of training rounds. 

% \newpage
% \subsection{Evolution of $L$-ensemble Kernel}
% \begin{figure}[htb!]
%   \begin{center}
%     \includegraphics[width=\textwidth]{figures/Fashion-MNIST-LR-N10-m5-T20000-stepsize0.001-batch_size16-non-iid-kernel-DPP-0.6.png}  
%   \end{center}
%     \caption{Evolution of $L$-ensemble kernel for \textsc{DPP-Norm}.}
%     \label{fig:heatmap:norm}
% \end{figure}  

% \begin{figure}[htb!]
%   \begin{center}
%     \includegraphics[width=\textwidth]{figures/Fashion-MNIST-LR-N10-m5-T20000-stepsize0.001-batch_size16-non-iid-kernel-DPP-Loss-0.6.png}  
%   \end{center}
%     \caption{Evolution of $L$-ensemble kernel for \textsc{DPP-Loss}.}
%     \label{fig:heatmap:loss}
% \end{figure}  

% \subsection{Evolution of $L$}
% We show the evolution of the kernel matrix $L$ during the entire training of our \textsc{ADCS} in Figure~\ref{dpp:figure:comparison_L}. We observe that gradients from the clients of the same cluster are indeed similar since the corresponding entries in $L$ are large. The entries correspond to clients from different clients can vary a lot. Besides, the entries of $L$ is not fixed and is gradually changing over time, indicating that the kernels are capturing the learning dynamics, probability from both the model and the data distributions. 

\subsection{Empirical Selection Bias}

% \begin{figure}[t]
% \centering
% \begin{subfigure}[b]{0.23\textwidth}
%     \centering
%     \includegraphics[width=\textwidth]{figures/Fashion-MNIST-LR-N30-m18-T20000-stepsize0.001-batch_size16-non-iid-copy-comparison-bias-smooth.png}
% \end{subfigure}
% \caption{Gradient approximation bias for varying $\theta$.}
% \label{dpp:figure:comparison_bias}
% \end{figure}

In Table~\ref{tab:selection_ratio}, we study the client selection bias by showing the ratio of each group of clients selected during training. To further quantify bias, we use the total variation (TV) distance. The TV distance between any two distributions $s, t \in \Delta_{K-1}$ is $\mathrm{TV}(s,t) = {\frac{1}{2}\sum_{k=1}^{K} |s_k-t_k|}$. Using the uniform target distribution as $t$ with ratio $0.1$ for each group, the TV distance is $0.0019$ for \textsc{Uniform}, $0.2241$ for \textsc{OCS}, $0.1934$ for \textsc{PowerOfChoice}, and $0.1364$ for \textsc{ADCS}. It is clear that uniform client selection is the least biased, but it does not perform well on the worst-case test accuracy as shown in Figures~\ref{dpp:figure:comparison},~\ref{dpp:figure:comparison_theta} and~\ref{dpp:figure:comparison_R}.

% In the theoretical analysis, we assume the gradient approximation error is bounded. The unnormalized bias is shown in Fig.~\ref{dpp:figure:comparison_bias} with different $\theta$. We observe that the gradient approximation error is indeed upper bounded, though there is no simple pattern in how $\theta$ affects the gradient approximation error empirically. 

% \begin{figure}[htb!]
% \centering
% % \begin{subfigure}
% %     \centering
% %     \includegraphics[width=0.45\textwidth]{figures/Fashion-MNIST-LR-N30-m18-T20000-stepsize0.001-batch_size16-non-iid-copy-comparison-bias.png}
% % \end{subfigure}
% \begin{subfigure}
%     \centering
%     \includegraphics[width=0.45\textwidth]{figures/Fashion-MNIST-LR-N30-m18-T20000-stepsize0.001-batch_size16-non-iid-copy-comparison-bias-smooth.png}
% \end{subfigure}
% \caption{Comparison of gradient approximation error with different $\theta$ for AdaDPP.}
% \label{fig:DPP:bias:unnormalized}
% \end{figure}

\section{Conclusion}\label{sec:conclusion}
In this work, we have studied the client scheduling problem in FL from a joint diversity-quality perspective. Departing from existing approaches that rely solely on additive quality scores or a pure diversity objective, we propose a principled framework based on DPPs to explicitly capture correlations and redundancy among clients. We introduce \textsc{ADCS}, a practical federated learning algorithm that performs client scheduling via MAP inference of an $m$-DPP.  We investigate the upper bound of selection bias in general and specifically for \textsc{ADCS} under a clustering structure of gradients. We further derive a convergence bound for biased client scheduling for general non-convex functions, giving an $\varepsilon$-dependent stationarity bound for \textsc{ADCS}. Experiments on non-IID federated learning demonstrated that \textsc{ADCS} consistently improves performance on worst-off clients, compared with state-of-the-art client scheduling methods. Our results further highlight the importance of diversity-aware selection, particularly in heterogeneous data settings. Together, these findings suggest that balancing client quality and diversity offers a principled path toward FL that better serves its worst-performing clients while preserving competitive average performance.

% \section*{Acknowledgments} 

\bibliographystyle{plainnat}
\bibliography{references}

\newpage
\appendix
\section{Proof of Theorem~\ref{dpp:thm:cluster_adaDPP}}
\label{dpp:appendix:thm1}
\begin{proof}
We denote the $L$-ensemble kernel constructed by \textsc{ADCS} be $L_{\text{ADCS}}$. For every subset \(S\subseteq[N]\),
\begin{equation}
\det((L_{\text{ADCS}})_S)
=
\left(\prod_{i\in S} w_i\right)\det(C_S).
\label{eq:det-factorization-corrected}
\end{equation}
We define normalized $u_k \triangleq \frac{c_k}{\|c_k\|_2}$. We divide the proof into four steps. 

\medskip
\noindent
\textbf{Step 1: Every subset with clients from a duplicate cluster has small determinant.}
Take any subset \(S\subseteq[N]\) with \(|S|=m\), and suppose \(S\) contains two distinct clients from the same cluster. After relabeling the selected indices, we can write $S=\{i_1,\dots,i_m\}$
with $z_{i_1}=z_{i_2}$. Let $X_S = [\phi_{i_1},\dots,\phi_{i_m}] \in \mathbb{R}^{d\times m}$, then $C_S = X_S^\top X_S$. From Assumption~\ref{dpp:assumption:cluster}, we have
\begin{align}
\left\|\phi_i-u_k\right\|_2
=
\left\|
\frac{g_i}{\|g_i\|_2}-\frac{c_k}{\|c_k\|_2}
\right\|_2
\le
\frac{2\|g_i-c_k\|_2}{\|c_k\|_2}
\le
\frac{2r}{\kappa}
\triangleq \delta
\end{align}
for every \(i\in C_k\). Since \(z_{i_1}=z_{i_2}\), both \(\phi_{i_1}\) and \(\phi_{i_2}\) lie within distance \(\delta\) of the same center \(u_{z_{i_1}}\). Hence, by the triangle inequality,
\begin{align}
\|\phi_{i_2}-\phi_{i_1}\|_2
\le
\|\phi_{i_2}-u_{z_{i_1}}\|_2
+
\|\phi_{i_1}-u_{z_{i_1}}\|_2
\le
2\delta.
\end{align}

Now define the elementary matrix $M = I_m - E_{1,2}$, where $E_{1,2}$ has a $1$ in position $(1,2)$ and zeros elsewhere. Then $\det(M)=1$. We set $Y_S = X_S M$, where the second column of \(Y_S\) is $(Y_S)_{:,2} = \phi_{i_2}-\phi_{i_1}$ 
while every other column of \(Y_S\) equals one of the original selected vectors. Therefore
\begin{align}
\det(C_S)
=
\det(X_S^\top X_S)
% =
% \det(M^\top X_S^\top X_S M)
=
\det(Y_S^\top Y_S)
\stackrel{(a)}{\leq}
\prod_{j=1}^m \|(Y_S)_{:,j}\|_2^2
\stackrel{(b)}{\leq}
4\delta^2,
\end{align}
where $(a)$ is by the Hadamard's inequality, $(b)$ is by the fact that all unchanged columns have norm \(1\) since \(\|\phi_i\|_2=1\) and the modified second column satisfies $\|(Y_S)_{:,2}\|_2
= \|\phi_{i_2}-\phi_{i_1}\|_2 \le 2\delta$.

Using \eqref{eq:det-factorization-corrected} and the bound \(w_i\le \bar w\), we obtain
\begin{equation}
\det\!\big((L_{\text{ADCS}})_S\big)
\le
4\delta^2\,\bar w^{\,m}.
\label{eq:duplicate-upper-corrected}
\end{equation}

\medskip
\noindent
\textbf{Step 2: There exists a distinct-cluster subset with larger determinant.}
Choose $T^\sharp \in \arg\max_{\substack{T\subseteq[K]\\ |T|=m}}
\prod_{k\in T}\bar w_k$, so that $\prod_{k\in T^\sharp}\bar w_k = Q_m$. For each \(k\in T^\sharp\), choose a client \(i_k\in \mathcal C_k\) such that $w_{i_k} = \bar w_k$. Let $S^\sharp := \{i_k:k\in T^\sharp\}$. By construction, \(S^\sharp\) contains one client from each of \(m\) distinct clusters. Let
\begin{align}
& X^\sharp \triangleq [\phi_{i_k}]_{k\in T^\sharp},  \\
& U^\sharp \triangleq [u_k]_{k\in T^\sharp}, \\
& E^\sharp \triangleq X^\sharp - U^\sharp.
\end{align}
Each column of \(E^\sharp\) has norm at most \(\delta\), so $\|E^\sharp\|_2
\le \|E^\sharp\|_F \le \sqrt{m}\,\delta$. Consider the Gram matrix \((U^\sharp)^\top U^\sharp\). Its diagonal entries are \(1\), and its off-diagonal entries satisfy
\begin{align}
|(u_k)^\top u_\ell| \le \mu,
\qquad \forall (k\neq \ell).
\end{align}
By Gershgorin's circle theorem,
\begin{align}
\lambda_{\min}\big((U^\sharp)^\top U^\sharp\big)
\ge
1-(m-1)\mu.
\end{align}
Hence
\begin{align}
\sigma_{\min}(U^\sharp)
=
\sqrt{\lambda_{\min}\big((U^\sharp)^\top U^\sharp\big)}
\ge
\sqrt{1-(m-1)\mu}.
\end{align}

By Weyl's inequality for singular values,
\begin{align}
\sigma_j(X^\sharp)
\ge
\sigma_j(U^\sharp)-\|E^\sharp\|_2
\ge
\sigma_{\min}(U^\sharp)-\sqrt{m}\,\delta
\ge
\rho,
\end{align}
for all $j \in [m]$. Since \(\rho>0\), this yields
\begin{align}
\det(C_{S^\sharp})
=
\det\!\big((X^\sharp)^\top X^\sharp\big)
=
\prod_{j=1}^m \sigma_j(X^\sharp)^2
\ge
\rho^{2m}.
\end{align}

\textbf{Step 3: Comparing the two determinants.} Combining this with \eqref{eq:det-factorization-corrected},
\begin{align}
\det\!\big(L_{S^\sharp}\big)
=
\left(\prod_{k\in T^\sharp} w_{i_k}\right)\det(C_{S^\sharp})
\ge
\left(\prod_{k\in T^\sharp}\bar w_k\right)\rho^{2m}
=
Q_m\rho^{2m}.
\end{align}
From the condition in~\eqref{eqn:assumption:intra-cluster-concentration} and recall $\delta = 2r/\kappa$, we can show that
\begin{align}
    Q_m\rho^{2m}>16r^2\bar{w}^{m}/\kappa^2,
\label{eqn:assumption:thm1_condition}
\end{align}
Let $F(r)\triangleq Q_m(\sqrt{1-(m-1)\mu}-2r\sqrt {m}/\kappa)^{2m}-16r^2\bar w^{\,m}/\kappa^2$. Since \(Q_m>0\), we have \(F(0)=Q_m \left(\sqrt{1-(m-1)\mu}\right)^{2m}>0\). By continuity of \(F\), there exists \(r_\star>0\) such that $F(r_\star) = 0$, or otherwise $F(\cdot) > 0$ always holds for $r \geq 0$, in which case $r_\star = +\infty$. Thus, we obtain
\begin{align}\label{DPP:eqn:thm1:step2}
\det\big(L_{S^\sharp}\big)
>
\det\big(L_S\big).
\end{align}
We conclude that every subset containing a duplicate cluster has strictly smaller determinant than the particular distinct-cluster subset \(S^\sharp\). Therefore, no MAP maximizer can contain clients from a duplicate cluster. Since \(|S_{\text{ADCS}}^\star|=m\), it follows that every exact MAP solution selects \(m\) distinct clusters.

\textbf{Step 4: Bounding approximation error of $\bar g_{S_{\text{ADCS}}^\star}$.}
The set $S_{\text{ADCS}}^\star$ contains exactly one client from each cluster in \(T_{\text{ADCS}}^\star\), and none from the remaining clusters. Let $\epsilon_{i} \triangleq g_i - c_{z_i}$. We have $\|\epsilon_{i}\| \leq r$. Therefore,
\begin{align}
\bar g_{S_{\text{ADCS}}^\star}
& =
\frac1m\sum_{k\in T_{\text{ADCS}}^\star} c_k
+
\frac1m\sum_{i\in S_{\text{ADCS}}^\star}\varepsilon_i \nonumber \\
& =
\sum_{k=1}^K \frac{\mathbf{1}\{k\in T_{\text{ADCS}}^\star\}}{m} c_k
+
\frac1m\sum_{i\in S_{\text{ADCS}}^\star}\varepsilon_i.
\end{align}
Also,
\begin{align}
\bar g_N
=
\sum_{k=1}^K \frac{n_k}{N} c_k
+
\frac1N\sum_{i=1}^N \varepsilon_i.
\end{align}
Therefore,
\begin{align}
& \bar g_{S_{\text{ADCS}}^\star}-\bar g_N =
\sum_{k=1}^K
\left(
\frac{\mathbf{1}\{k\in T_{\text{ADCS}}^\star\}}{m}
-
\frac{n_k}{N}
\right)c_k
+
\left(
\frac1m\sum_{i\in S_{\text{ADCS}}^\star}\varepsilon_i
-
\frac1N\sum_{i=1}^N \varepsilon_i
\right).
\end{align}
Applying the triangle inequality and using \(\|\varepsilon_i\|_2\le r\) gives the bound.
\end{proof}

\vspace{2mm}
\section{Proof of Corollary~\ref{dpp:corollary:comparison}}
\label{dpp:appendix:cor1}
\begin{proof}
Since \(g_i=c_{z_i}\) and the vectors \(c_1,\dots,c_K\) are orthonormal, each \(c_k\) has unit norm. Hence $\phi_i = \frac{g_i}{\|g_i\|_2} = c_{z_i}$. Therefore Theorem~\ref{dpp:thm:cluster_adaDPP} applies with $u_k=c_k$, $r=0$, $\mu=0$. We have
\begin{align}
    \rho=\sqrt{1-(m-1)\mu}-2\sqrt{m}r/\kappa = 1.
\end{align}

Moreover, since \(w_i=e^{2\alpha q_i}>0\) for all \(i\), we have
\begin{align}
    Q_m\rho^{2m}=Q_m>0=16r^2\bar w^m/{\kappa^2},
\end{align}
so the strict inequality in Theorem~\ref{dpp:thm:cluster_adaDPP} is satisfied. Thus every exact MAP solution \(S_{\text{ADCS}}^\star\) still selects \(m\) distinct clusters.

Now apply Theorem~\ref{dpp:thm:cluster_adaDPP}. Since \(r=0\) and \(n_k=N/K\) for all \(k\),
\begin{align}
\|\bar g_{S_{\text{ADCS}}^\star}-\bar g_N\|_2
\le
\left\|
\sum_{k=1}^K
\left(
\frac{\mathbf{1}\{k\in T_{\text{ADCS}}^\star\}}{m}
-
\frac{1}{K}
\right)c_k
\right\|_2.
\end{align}
Because the vectors \(c_1,\dots,c_K\) are orthonormal and \(|T_{\text{ADCS}}^\star|=m\),
\begin{align}
\left\|
\sum_{k=1}^K
\left(
\frac{\mathbf{1}\{k\in T_{\text{ADCS}}^\star\}}{m}
-
\frac{1}{K}
\right)c_k
\right\|_2^2
=
\sum_{k=1}^K
\left(
\frac{\mathbf{1}\{k\in T_{\text{ADCS}}^\star\}}{m}
-
\frac{1}{K}
\right)^2.
\end{align}
There are exactly \(m\) indices in \(T_{\text{ADCS}}^\star\) and \(K-m\) indices outside \(T_{\text{ADCS}}^\star\), so
\begin{align}
& \sum_{k=1}^K
\left(
\frac{\mathbf{1}\{k\in T_{\text{ADCS}}^\star\}}{m}
-
\frac{1}{K}
\right)^2  =
m\left(\frac{1}{m}-\frac{1}{K}\right)^2
+
(K-m)\left(\frac{1}{K}\right)^2
=
\frac{1}{m}-\frac{1}{K}.
\end{align}
Hence
\begin{align}
\|\bar g_{S_{\text{ADCS}}^\star}-\bar g_N\|_2
\le
\sqrt{\frac{1}{m}-\frac{1}{K}}.
\end{align}

% Next consider the generic spectral upper bound. Since $G=[g_1,\dots,g_N]$ and $g_i=c_{z_i}$, we have
% \[
% GG^\top
% =
% \sum_{i=1}^N g_i g_i^\top
% =
% \sum_{k=1}^K n_k\, c_k c_k^\top
% =
% \frac{N}{K}\sum_{k=1}^K c_k c_k^\top.
% \]
% Because \(c_1,\dots,c_K\) are orthonormal, \(\sum_{k=1}^K c_k c_k^\top\) is the orthogonal projector onto \(\mathrm{span}\{c_1,\dots,c_K\}\), whose largest eigenvalue is \(1\). Therefore
% \[
% \lambda_{\max}(GG^\top)=\frac{N}{K}.
% \]
% Since \(L=G^\top G\) and \(GG^\top\) have the same nonzero eigenvalues,
% \[
% \lambda_{\max}(L)=\frac{N}{K}.
% \]
% Substituting into the generic spectral upper bound gives
% \[
% \|\bar g_{S^\star}-\bar g_N\|_2
% \le
% \sqrt{\lambda_{\max}(L)\left(\frac{1}{m}-\frac{1}{N}\right)}
% =
% \sqrt{\frac{N}{K}\left(\frac{1}{m}-\frac{1}{N}\right)}
% =
% \sqrt{\frac{N-m}{Km}}.
% \]

% Finally, if \(N>K\) and \(m<K\), then
% \[
% \frac{N-m}{Km}
% -
% \left(\frac{1}{m}-\frac{1}{K}\right)
% =
% \frac{N-m}{Km}-\frac{K-m}{Km}
% =
% \frac{N-K}{Km}
% >0.
% \]
% Hence
% \[
% \sqrt{\frac{1}{m}-\frac{1}{K}}
% <
% \sqrt{\frac{N-m}{Km}}.
% \]
% If instead \(m=K\), then the cluster-aware bound equals
% \[
% \sqrt{\frac{1}{K}-\frac{1}{K}}=0,
% \]
% while the generic spectral bound becomes
% \[
% \sqrt{\frac{N-K}{K^2}},
% \]
% which is strictly positive when \(N>K\).
\end{proof}

\vspace{2mm}
\section{Proof of Theorem~\ref{thm:nonconvex_dpp_fedavg}}
\label{dpp:appendix:thm2}
\begin{proof}
We consider the global objective $f(\model) = \frac{1}{N}\sum_{i=1}^N f_i(\model)$, whose gradient is $\nabla f(\model) = \frac{1}{N}\sum_{i=1}^N \nabla f_i(\model)$. At each communication round $t$, the server selects a subset $\cS^{(t)} \subseteq [N]$ of size $|\cS^{(t)}|=m$ via some client selection optimization.

At round $t$, the server broadcasts the global model $\modelatt$. Each selected client $i\in \cS^{(t)}$ runs $E$ local SGD steps with local stepsize $\eta_\ell>0$ after the initialization $\model^{(t,0)}_{i}=\modelatt$
\begin{equation}
\model^{(t,e+1)}_{i}=\model^{(t,e)}_{i}-\eta_\ell\, g^{(t,e)}_{i},
\quad \text{for all } e=0,1,\dots,E-1,
\label{eq:local_sgd}
\end{equation}
where $g^{(t,e)}_{i}$ is a stochastic gradient $\nabla f_i(\model^{(t,e)}_{i}; \xi^{(t,e)}_{i})$ using a mini-batch. From~\eqref{eq:local_sgd}, for any selected client $i\in \cS^{(t)}$,
\begin{equation}
\model^{(t,E)}_{i}
=
\modelatt-\eta_\ell\sum_{e=0}^{E-1} g^{(t,e)}_{i}.
\label{eq:local_sgd_endpoint}
\end{equation}
{For the client updates $g_i^{(t)}=\sum_{e=0}^{E-1}g_i^{(t,e)}$ and uniform client weights, the server update in Algorithm~\ref{alg:ADCS} is
\begin{align}
\modelattplusone
&=\modelatt-\frac{\eta_s}{m}\sum_{i\in\cS^{(t)}}\sum_{e=0}^{E-1}g_i^{(t,e)} \nonumber\\
&=\modelatt-\eta_s\sum_{e=0}^{E-1}\underbrace{\left(\frac{1}{m}\sum_{i\in \cS^{(t)}}g^{(t,e)}_{i}\right)}_{\triangleq \hat g_{t,e}}.
\label{eq:server_gradient_sum}
\end{align}
Defining $\eta\triangleq E\eta_s$ and $G_t\triangleq\frac{1}{E}\sum_{e=0}^{E-1}\hat g_{t,e}$ gives $\modelattplusone=\modelatt-\eta G_t$.}

By $L$-smoothness of $f$, for $\modelattplusone=\modelatt-\eta G_t$, we have
\begin{align}
f(\modelattplusone)
&\le
f(\modelatt) + \langle \nabla f(\modelatt), \modelattplusone-\modelatt\rangle
+ \frac{L}{2}\|\modelattplusone-\modelatt\|^2 \nonumber\\
&=
f(\modelatt) - \eta\langle \nabla f(\modelatt), G_t\rangle
+ \frac{L\eta^2}{2}\|G_t\|^2.
\end{align}
Taking conditional expectation given $\mathcal{F}_t$ yields
\begin{align}
& \mathbb{E}[f(\modelattplusone)\mid \mathcal{F}_t] \le
f(\modelatt)
\underbrace{-\eta\left\langle \nabla f(\modelatt), \mathbb{E}[G_t\mid \mathcal{F}_t]\right\rangle}_{\text{B1}}
+\underbrace{\frac{L\eta^2}{2}\mathbb{E}[\|G_t\|^2\mid \mathcal{F}_t]}_{\text{B2}}.
\label{eq:smooth_cond}
\end{align}

Now we start to bound $B1$. We first decompose $G_t$ into gradient, bias, drift, and zero-mean noise.
For each $e$, decompose $\hat g_{t,e}$ as
\begin{align}
\hat g_{t,e}
&=
\underbrace{\frac{1}{m}\sum_{i\in \cS^{(t)}}\nabla f_i(\modelatt)}_{\bar g_t}
+
\underbrace{\frac{1}{m}\sum_{i\in \cS^{(t)}}\Big(\nabla f_i(\model^{(t,e)}_{i})-\nabla f_i(\modelatt)\Big)}_{d_{t,e}} \nonumber \\
& \quad +
\underbrace{\frac{1}{m}\sum_{i\in \cS^{(t)}}\Big(g^{(t,e)}_{i}-\nabla f_i(\model^{(t,e)}_{i})\Big)}_{\xi_{t,e}}.
\label{eq:hatg_decomp}
\end{align}
Average over $e=0,\dots,E-1$:
\begin{equation}
G_t
=
\bar g_t + d_t + \xi_t,
\qquad
d_t\triangleq \frac{1}{E}\sum_{e=0}^{E-1}d_{t,e},
\quad
\xi_t\triangleq \frac{1}{E}\sum_{e=0}^{E-1}\xi_{t,e}.
\label{eq:Gt_decomp}
\end{equation}
Next, decompose $\bar g_t$ into its conditional mean and deviation:
\begin{equation}
\bar g_t
=
\mathbb{E}[\bar g_t\mid \mathcal{F}_t]+\zeta_t
=
\nabla f(\modelatt)+b_t+\zeta_t,
\end{equation}
where $b_t = \mathbb{E}[b_t \mid \mathcal{F}_t] = \mathbb{E}[\bar g_t\mid \mathcal{F}_t] - \nabla f(\modelatt)$ is the conditional bias and $\zeta_t = \bar g_t-\mathbb{E}[\bar g_t\mid \mathcal{F}_t]$ is the zero-mean selection deviation. Clearly, $\mathbb{E}[\zeta_t\mid \mathcal{F}_t]=0$.
Combining with~\eqref{eq:Gt_decomp},
\begin{equation}
G_t
=
\nabla f(\modelatt)+b_t+d_t+u_t,
\label{eq:Gt_master}
\end{equation}
where $u_t=\zeta_t+\xi_t$ and $\mathbb{E}[u_t\mid \mathcal{F}_t]=0$. From~\eqref{eq:Gt_master}, we obtain
\begin{align}
\mathbb{E}[G_t\mid \mathcal{F}_t]
=
\nabla f(\modelatt)+b_t+\mathbb{E}[d_t\mid \mathcal{F}_t].
\end{align}
Hence
\begin{align}
&-\eta\left\langle \nabla f(\modelatt), \mathbb{E}[G_t\mid \mathcal{F}_t]\right\rangle =
-\eta\|\nabla f(\modelatt)\|^2
-\eta\langle \nabla f(\modelatt), b_t\rangle
-\eta\left\langle \nabla f(\modelatt), \mathbb{E}[d_t\mid \mathcal{F}_t]\right\rangle. 
\end{align}
Apply Young's inequality $\langle a,c\rangle \le \frac{1}{4}\|a\|^2+\|c\|^2$ twice:
\begin{align}
-\eta\langle \nabla f(\modelatt), b_t\rangle
&\le
\eta\cdot\frac{1}{4}\|\nabla f(\modelatt)\|^2 + \eta\|b_t\|^2, \\
-\eta\left\langle \nabla f(\modelatt), \mathbb{E}[d_t\mid \mathcal{F}_t]\right\rangle
&\le
\eta\cdot\frac{1}{4}\|\nabla f(\modelatt)\|^2 + \eta\|\mathbb{E}[d_t\mid \mathcal{F}_t]\|^2.
\end{align}
Therefore
\begin{align}
& -\eta\left\langle \nabla f(\modelatt), \mathbb{E}[G_t\mid \mathcal{F}_t]\right\rangle
\nonumber \\
& \le
-\frac{\eta}{2}\|\nabla f(\modelatt)\|^2
+\eta\|b_t\|^2
+\eta\|\mathbb{E}[d_t\mid \mathcal{F}_t]\|^2 \nonumber \\
& \le
-\frac{\eta}{2}\|\nabla f(\modelatt)\|^2
+\eta\|b_t\|^2
+\eta\,\mathbb{E}[\|d_t\|^2\mid \mathcal{F}_t],
\label{eq:progress_bound}
\end{align}
where the last inequality is by Jensen's inequality $\|\mathbb{E}[d_t\mid \mathcal{F}_t]\|^2 \le \mathbb{E}[\|d_t\|^2\mid \mathcal{F}_t]$.

Now we start to bound B2. Using $\|a+b+c+d\|^2 \le 4(\|a\|^2+\|b\|^2+\|c\|^2+\|d\|^2)$ in~\eqref{eq:Gt_master}, we obtain
\begin{align}
\|G_t\|^2
\le
4\Big(\|\nabla f(\modelatt)\|^2+\|b_t\|^2+\|d_t\|^2+\|u_t\|^2\Big).
\end{align}
Thus
\begin{align}
\mathbb{E}[\|G_t\|^2\mid \mathcal{F}_t] \le
4\|\nabla f(\modelatt)\|^2
+4\|b_t\|^2
+4\mathbb{E}[\|d_t\|^2\mid \mathcal{F}_t]
+4\mathbb{E}[\|u_t\|^2\mid \mathcal{F}_t].
\label{eq:norm_bound:prelim}
\end{align}
Now bound $\mathbb{E}[\|u_t\|^2\mid \mathcal{F}_t]$.
Since $u_t=\zeta_t+\xi_t$,
\begin{align}
\mathbb{E}[\|u_t\|^2\mid \mathcal{F}_t]
\le
2\mathbb{E}[\|\zeta_t\|^2\mid \mathcal{F}_t]
+
2\mathbb{E}[\|\xi_t\|^2\mid \mathcal{F}_t].
\end{align}
By the bound on selection randomness variance, $\mathbb{E}[\|\zeta_t\|^2\mid \mathcal{F}_t]\le \nu_{\rm sel}$.
For $\xi_t=\frac{1}{E}\sum_{e=0}^{E-1}\xi_{t,e}$ and given {Assumption~\ref{DPP:assumption:variance}}, a standard averaging bound gives $\mathbb{E}[\|\xi_{t,e}\|^2\mid \mathcal{F}_t]\le \frac{\sigma^2}{m}$ and $\mathbb{E}[\|\xi_t\|^2\mid \mathcal{F}_t]\le \frac{\sigma^2}{mE}$. Hence
\begin{align}
\mathbb{E}[\|u_t\|^2\mid \mathcal{F}_t]
\le
2\nu_{\rm sel}+\frac{2\sigma^2}{mE}.
\label{eq:u_var_bound}
\end{align}

Next, we show
\begin{equation}
\mathbb{E}[\|d_t\|^2\mid \mathcal{F}_t]
\le
\frac{L^2\eta_\ell^2 G^2 (E-1)(2E-1)}{6}.
\label{eq:drift_bound}
\end{equation}
By $L$-smoothness,
\begin{equation}
\|\nabla f_i(\model^{(t,e)}_{i})-\nabla f_i(\modelatt)\|
\le L\|\model^{(t,e)}_{i}-\modelatt\|.
\end{equation}
We obtain
\begin{equation}
\|d_{t,e}\|
\le
\frac{1}{m}\sum_{i\in \cS^{(t)}} L\|\model^{(t,e)}_{i}-\modelatt\|.
\end{equation}
Using Cauchy–Schwarz inequality $(\frac1m\sum a_i)^2\le \frac1m\sum a_i^2$,
\begin{equation}
\|d_{t,e}\|^2
\le
L^2\cdot \frac{1}{m}\sum_{i\in \cS^{(t)}}\|\model^{(t,e)}_{i}-\modelatt\|^2.
\end{equation}
Moreover, from~\eqref{eq:local_sgd},
\begin{equation}
\model^{(t,e)}_{i}-\modelatt
=
-\eta_\ell\sum_{j=0}^{e-1} g^{(i)}_{t,j},
\end{equation}
and by Cauchy-Schwarz inequality,
\begin{equation}
\|\model^{(t,e)}_{i}-\modelatt\|^2
\le
\eta_\ell^2\, e\sum_{j=0}^{e-1}\|g^{(i)}_{t,j}\|^2.
\end{equation}
Taking conditional expectation and applying {Assumption~\ref{DPP:assumption:bounded2moment}} yields
\begin{equation}
\mathbb{E}[\|\model^{(t,e)}_{i}-\modelatt\|^2\mid \mathcal{F}_t]
\le
\eta_\ell^2\, e\cdot e\, G^2
=
\eta_\ell^2 e^2 G^2.
\end{equation}
Therefore, $\mathbb{E}[\|d_{t,e}\|^2\mid \mathcal{F}_t]
\le
L^2\eta_\ell^2 e^2 G^2$. Finally, by $d_t=\frac{1}{E}\sum_{e=0}^{E-1}d_{t,e}$ and Jensen inequality, we have
\begin{align}
\mathbb{E}[\|d_t\|^2\mid \mathcal{F}_t]
& \le
\frac{1}{E}\sum_{e=0}^{E-1}\mathbb{E}[\|d_{t,e}\|^2\mid \mathcal{F}_t] \nonumber \\
& \le
\frac{L^2\eta_\ell^2 G^2}{E}\sum_{e=0}^{E-1}e^2 \nonumber \\
& = \frac{L^2\eta_\ell^2 G^2 (E-1)(2E-1)}{6}. 
\end{align}
and
\begin{align}
\mathbb{E}[\|G_t\|^2\mid \mathcal{F}_t]
& \le
4\|\nabla f(\modelatt)\|^2
+4\|b_t\|^2
+\frac{2L^2\eta_\ell^2 G^2 (E-1)(2E-1)}{3}
+8\nu_{\rm sel}+{\frac{8\sigma^2}{mE}}.
\label{eq:norm_bound}
\end{align}

Finally substituting the bound for B1 in~\eqref{eq:progress_bound} and the bound for B2 in~\eqref{eq:norm_bound} into~\eqref{eq:smooth_cond}, we obtain
\begin{align}
\mathbb{E}[f(\modelattplusone)\mid \mathcal{F}_t]
&\le
f(\modelatt)
-\frac{\eta}{2}\|\nabla f(\modelatt)\|^2
+\eta\|b_t\|^2
+\eta\,\mathbb{E}[\|d_t\|^2\mid \mathcal{F}_t] \nonumber\\
&\quad
+\frac{L\eta^2}{2}\Big(
4\|\nabla f(\modelatt)\|^2
+4\|b_t\|^2
+4\mathbb{E}[\|d_t\|^2\mid \mathcal{F}_t]
+4\mathbb{E}[\|u_t\|^2\mid \mathcal{F}_t]
\Big) \nonumber\\
&=
f(\modelatt)
-\left(\frac{\eta}{2}-2L\eta^2\right)\|\nabla f(\modelatt)\|^2
+\left(\eta+2L\eta^2\right)\|b_t\|^2 \nonumber\\
&\quad
+\left(\eta+2L\eta^2\right)\mathbb{E}[\|d_t\|^2\mid \mathcal{F}_t]
+2L\eta^2\,\mathbb{E}[\|u_t\|^2\mid \mathcal{F}_t].
\label{eq:one_step}
\end{align}
Assuming $\eta\le \frac{1}{8L}$. Then we have
\begin{align}
\frac{\eta}{2}-2L\eta^2 \ge \frac{\eta}{4}, \label{DPP:thm2:proof:eta1}
\end{align}
and
\begin{align}
\eta+2L\eta^2 \le \frac{5\eta}{4}.
\label{DPP:thm2:proof:eta2}
\end{align}
Applying \eqref{DPP:thm2:proof:eta1} and \eqref{DPP:thm2:proof:eta2} to~\eqref{eq:one_step} gives
\begin{align}
\mathbb{E}[f(\modelattplusone)\mid \mathcal{F}_t]
& \le
f(\modelatt)
-\frac{\eta}{4}\|\nabla f(\modelatt)\|^2
+\frac{5\eta}{4}\|b_t\|^2
+\frac{5\eta}{4}\mathbb{E}[\|d_t\|^2\mid \mathcal{F}_t] \nonumber \\
& \quad +2L\eta^2\,\mathbb{E}[\|u_t\|^2\mid \mathcal{F}_t].
\end{align}
Taking full expectation and using the bound on expected bias,~\eqref{eq:drift_bound}, and~\eqref{eq:u_var_bound}, we obtain
\begin{align}
\mathbb{E}[f(\modelattplusone)] 
&\le
\mathbb{E}[f(\modelatt)]
-\frac{\eta}{4}\mathbb{E}\|\nabla f(\modelatt)\|^2
+\frac{5\eta}{4}\varepsilon^2  \nonumber \\
& \quad +\frac{5\eta}{4}\cdot \frac{L^2\eta_\ell^2 G^2 (E-1)(2E-1)}{6}
+2L\eta^2\left(2\nu_{\rm sel}+\frac{2\sigma^2}{mE}\right) \nonumber\\
&=
\mathbb{E}[f(\modelatt)]
-\frac{\eta}{4}\mathbb{E}\|\nabla f(\modelatt)\|^2
+\frac{5\eta}{4}\varepsilon^2  \nonumber \\
& \quad +\frac{5}{24}\eta L^2\eta_\ell^2 G^2 (E-1)(2E-1)
+4L\eta^2\left(\nu_{\rm sel}+\frac{\sigma^2}{mE}\right).
\label{eq:recursion}
\end{align}

Summing~\eqref{eq:recursion} over $t=0,1,\dots,T-1$, we have
\begin{align}
\mathbb{E}[f(\modelatT)]
& \le
f(\model_0)
-\frac{\eta}{4}\sum_{t=0}^{T-1}\mathbb{E}\|\nabla f(\modelatt)\|^2
 \nonumber \\
& \quad +T\bigg(\frac{5\eta}{4}\varepsilon^2
+\frac{5}{24}\eta L^2\eta_\ell^2 G^2 (E-1)(2E-1)  +4L\eta^2\left(\nu_{\rm sel}+\frac{\sigma^2}{mE}\right)
\bigg). 
\end{align}
Since $f(\modelatT)\ge f^\star$, we have
\begin{align}
\frac{\eta}{4}\sum_{t=0}^{T-1}\mathbb{E}\|\nabla f(\modelatt)\|^2 
& \le
f(\model_0)-f^\star + 
\frac{5}{24}\eta L^2\eta_\ell^2 G^2 (E-1)(2E-1)T
\nonumber \\
& \quad +\left(
\frac{5\eta}{4}\varepsilon^2
+4L\eta^2\left(\nu_{\rm sel}+\frac{\sigma^2}{mE}\right)
\right)T. 
\end{align}
Divide both sides by $(\eta/4)T$ to obtain
\begin{align}
\frac{1}{T}\sum_{t=0}^{T-1}\mathbb{E}\|\nabla f(\modelatt)\|^2
& \le
\frac{4\big(f(\model_0)-f^\star\big)}{\eta T}
+\frac{5}{6}L^2\eta_\ell^2 G^2 (E-1)(2E-1)
\nonumber \\
& \quad +5\varepsilon^2+16L\eta\left(\nu_{\rm sel}+\frac{\sigma^2}{mE}\right), 
\end{align}
which completes the proof.
\end{proof}

% \section{Extra Experiments}

% \begin{figure*}[htb!]
% \centering
% \begin{subfigure}
%     \centering
%     \includegraphics[width=0.32\textwidth]{figures/Imagenette_160px-LR-N30-m10-T20000-stepsize0.001-batch_size16-non-iid-copy-comparison-test_accuracy.png}
% \end{subfigure}
% \begin{subfigure}
%     \centering
%     \includegraphics[width=0.32\textwidth]{figures/Imagenette_160px-LR-N30-m10-T20000-stepsize0.001-batch_size16-non-iid-copy-comparison-worst_test_accuracy.png}
% \end{subfigure}
% \begin{subfigure}
%     \centering
%     \includegraphics[width=0.32\textwidth]{figures/Imagenette_160px-LR-N30-m10-T20000-stepsize0.001-batch_size16-non-iid-copy-comparison-worst_test_accuracy20_.png}
% \end{subfigure}
% \caption{Comparison of average, worst, worst $20\%$ test accuracy on Imagenette with $m=10$.}
% \end{figure*}

% \begin{figure*}[htb!]
% \centering
% \begin{subfigure}
%     \centering
%     \includegraphics[width=0.32\textwidth]{figures/CIFAR10-LR-N30-m10-T20000-stepsize0.001-batch_size16-non-iid-copy-comparison-test_accuracy.png}
% \end{subfigure}
% \begin{subfigure}
%     \centering
%     \includegraphics[width=0.32\textwidth]{figures/CIFAR10-LR-N30-m10-T20000-stepsize0.001-batch_size16-non-iid-copy-comparison-worst_test_accuracy.png}
% \end{subfigure}
% \begin{subfigure}
%     \centering
%     \includegraphics[width=0.32\textwidth]{figures/CIFAR10-LR-N30-m10-T20000-stepsize0.001-batch_size16-non-iid-copy-comparison-worst_test_accuracy20_.png}
% \end{subfigure}
% \caption{Comparison ofaverage, worst, worst $20\%$ test accuracy on CIFAR with $m=10$.}
% \end{figure*}

\end{document}